\documentclass[letterpaper]{article} 
\usepackage{sty/aaai23}  
\usepackage{times}  
\usepackage{helvet}  
\usepackage{courier}  
\usepackage[hyphens]{url}  
\usepackage{graphicx} 
\usepackage{natbib}  
\usepackage{caption} 
\usepackage{algorithm}
\usepackage[noEnd=true,indLines=true,commentColor=black]{sty/algpseudocodex}

\usepackage{newfloat}
\usepackage{listings}
\DeclareCaptionStyle{ruled}{labelfont=normalfont,labelsep=colon,strut=off} 
\floatstyle{ruled}
\newfloat{listing}{tb}{lst}{}
\floatname{listing}{Listing}
\nocopyright

\usepackage{booktabs}
\usepackage{bm}
\usepackage{amsmath,amssymb,amsthm}

\newtheorem{observation}{Observation}
\newtheorem{theorem}{Theorem}
\newtheorem*{theorem*}{Theorem}
\newenvironment{sketch}{\proof}{\endproof}
\usepackage{hhline}
\usepackage[font=footnotesize,subrefformat=parens]{subcaption}
\usepackage{arydshln}
\usepackage{multirow}
\usepackage{tikz}
\usetikzlibrary{calc}
\usetikzlibrary{positioning}
\usetikzlibrary{shapes.geometric}
\usetikzlibrary{patterns}
\usetikzlibrary{decorations.pathreplacing}
\usetikzlibrary{fit}
\usetikzlibrary{arrows.meta}
\usetikzlibrary{bending}
\tikzset{
  vertex/.style={circle,draw,black,align=center,inner sep=0cm, minimum size=0.15cm,fill=black,anchor=center},
  line/.style={black},
  start-vertex/.style={vertex, minimum size=0.4cm,fill=white},
  goal-vertex/.style={vertex, minimum size=0.4cm,fill=white,rectangle},
}

\usepackage[binary-units]{siunitx}
\usepackage{cleveref}
\usepackage{sty/mystyle}
\title{Fault-Tolerant Offline Multi-Agent Path Planning}
\author{
  Keisuke Okumura,\textsuperscript{\rm 1}
  S{\'e}bastien Tixeuil\textsuperscript{\rm 2}
}
\affiliations{
  \textsuperscript{\rm 1} Tokyo Institute of Technology, Japan\\
  \textsuperscript{\rm 2} Sorbonne University, CNRS, LIP6, Institut Universitaire de France, France\\
  okumura.k@coord.c.titech.ac.jp, Sebastien.Tixeuil@lip6.fr
}

\begin{document}

\maketitle
\begin{abstract}
  We study a novel graph path planning problem for multiple agents that may crash at runtime, and block part of the workspace.
  In our setting, agents can detect neighboring crashed agents, and change followed paths at runtime.
  The objective is then to prepare a set of paths and switching rules for each agent, ensuring that all correct agents reach their destinations without collisions or deadlocks, despite unforeseen crashes of other agents.
  Such planning is attractive to build reliable multi-robot systems.
  We present problem formalization, theoretical analysis such as computational complexities, and how to solve this offline planning problem.
\end{abstract}

\section{Introduction}
Building robust and resilient multi-robot systems is an emerging and important topic~\cite{prorok2021beyond}, since those systems are expected to be infrastructures of logistics or product lines, as seen in fleet operations in automated warehouses~\cite{wurman2008coordinating}.
One fundamental problem in multi-robot systems is multi-agent path planning (MAPP), which assigns a collision/deadlock-free path to each agent.
Therefore, designing robust and resilient approaches to MAPP is a critical component to realize reliable multi-robot systems.

\emph{Robot faults} are not rare in practice due to sensor/motor errors, battery consumption, or other unexpected events.
Here, ``faults'' are beyond delays of robot motions as studied in the MAPP literature~\cite{atzmon2020robust,shahar2021safe,okumura2022offline}.
Rather, we consider them as critical events, e.g., agents forever stop their motions due to crashes, or, misbehave from the planning.
Then, fault-tolerant properties are essential for building reliable multi-robot systems, especially in lifelong scenarios involving a large number of robots.
Nevertheless, the cutting-edge MAPP studies largely overlook this aspect and assume that agents perfectly follow the offline planning that is prepared without any fault.

To this end, as the first step of fault-tolerant MAPP, \emph{we study an MAPP problem where agents may unexpectedly crash at runtime}.
The crashed agents then forever block part of the workspace.
Correct agents (i.e., non-crashed ones) can detect crashes through \emph{local observations} and then switch their executing path on the fly, based on this observation.
Our objective is to find a set of paths and their switching rules for each agent, such that correct agents can reach their destinations regardless of crash patterns.
We refer to the corresponding offline planning problem as \emph{MAPP with Crash Faults (MAPPCF)}.

Throughout the paper, we rely on local observations (rather than global observations) that permit to immediately detect a crash if it occurs on a neighboring location.
Therefore, significantly different from conventional MAPP studies, the challenge is to design a safe planning methodology that avoids collisions and deadlocks, under the assumption that agents behave following their plan, and their own observed information about other agents' crashes.

The contributions of this paper are twofold:
\paragraph{We formalize and analyze MAPPCF.}
Our formalization includes the conventional synchronous execution model of multi-agent pathfinding (MAPF)~\cite{stern2019def} where all agents take actions simultaneously, as well as the recently proposed asynchronous execution model called OTIMAPP~\cite{okumura2022offline}.
In our model, agents can switch executing paths on the fly according to local observation results.
The observation is done using a \emph{failure detector}, a black box function that tells an agent whether an adjacent location is occupied by a crashed agent, occupied by a correct agent, or vacant.
We consider anonymous and named failure detectors; the former cannot identify a crashed agent (only a crashed location).
After characterizing relationships between execution model and failure detector variants, we analyze the computational complexities of MAPPCF.
Our main results are that finding a solution is NP-hard, and verifying a solution is co-NP-complete.

\paragraph{We propose a methodology to solve MAPPCF.}
The proposed method, \emph{decoupled crash faults resolution framework (\algoname)}, resolves the effects of crashes one by one by preparing backup paths.
We evaluate \algoname with named failure detectors in grid environments and observe that \algoname can address more problem instances compared to computing a set of vertex disjoint paths, i.e., a trivially fault-tolerant approach since correct agents can reach their destinations regardless of crash patterns.
We further observe that the difficulty of finding solutions stems both from the problem instance size (e.g., the number of agents), and from the number of crashes to be tolerated.

\medskip
The paper organization is as follows.
\Cref{sec:problem-definition} formalizes MAPPCF.
\Cref{sec:prelim} and \ref{sec:complexity} present preliminary and computational complexity analysis, respectively.
\Cref{sec:solver} describes \algoname.
\Cref{sec:evaluation} presents empirical results obtained with \algoname.
\Cref{sec:related-work} reviews related work.
\Cref{sec:conclusion} concludes the paper with discussion of future directions.
The supplementary material is available on \url{https://kei18.github.io/mappcf}.
This paper uses ``MAPP'' as a generalized term for path planning for multiple agents, not limited to the formalization of classical MAPF.

\section{Problem Definition}
\label{sec:problem-definition}
\paragraph{MAPPCF Instance}
An \emph{MAPPCF instance} is given by a graph $G = (V, E)$, a set of agents $A=\{1, 2, \ldots, n\}$, the maximum number of crashes $f \in \mathbb{N}_{\geq 0}$, a tuple of starts $(s_1, s_2, \ldots, s_n)$, and goals $(g_1, g_2, \ldots, g_n)$, where $s_i, g_i \in V$ and for all $i \neq j$, $s_i \neq s_j$ and $g_i \neq g_j$.
An MAPPCF instance on digraphs is similar to the undirected case.

\paragraph{Plan}
A \emph{plan} for one agent comprises a list of paths each defined on $G$ and \emph{transition rules}.
At runtime, the agent moves along one path, called \emph{executing path}, in the plan, while always occupying one vertex.
Meanwhile, the agent switches its executing path following the transition rules.
A plan contains one special path called \emph{primary path} which is initially executed.
A transition rule is defined with a \emph{failure detector} and \emph{progress index}, explained below.

\paragraph{Crash and Failure Detector}
During plan execution, agents are potentially crashed.
\emph{Crashed} agents eternally remain in their occupying vertices.
We refer \emph{correct} agents to those who are not crashed.
Correct agents cannot pass where crashed agents are located.
However, a correct agent can use a \emph{failure detector} at runtime to change its executing path.
Doing so enables the correct agent to reach its goal under crash faults.
A failure detector tells a correct agent about the existence of an agent on an adjacent vertex, and if so, whether it has crashed or not.
We consider two types of detectors.
A detector is called \emph{named (NFD)} when it can identify who is crashed, otherwise \emph{anonymous (AFD)}.
Formal definitions are as follows.
Assume that an agent $i$ is at $v \in V$.
Let denote \neigh{v} a set of adjacent vertex of $v$.
Then, $\AFD:\neigh{v} \mapsto \{ \circ, \times, \bot \}$ and $\NFD:\neigh{v} \mapsto \{ \circ, \bot \} \cup A$.
Here, $\circ$ and $\times$ respectively correspond to a correct or crashed agent, otherwise $\bot$ is returned (no agent is there).
NFD returns an agent instead of $\times$.

\paragraph{Progress Index}
We use a \emph{progress index} $\clock{i} \in \mathbb{N}_{>0}$;
the agent $i$ is at $\langle\clock{i}\rangle$--th vertex of its executing path.
For each transition of executing paths, the progress index is initialized with one.
It increases up to the length of the executing path.

\paragraph{Transition Rule}
The rule specifies the next executing path given the current executing path, progress index, and results of the failure detector.
We then describe two execution models that differ in how to increment progress indexes.

\paragraph{Synchronous Execution Model (SYN)}
In this model, all agents take actions simultaneously.
More precisely, all the correct agents perform the following at the same time:
\emph{(i)}~may crash,
\emph{(ii)}~change executing paths if necessary,
\emph{(iii)}~move to their next vertices, and
\emph{(iv)}~increment their progress indexes.
Two types of collisions must be prohibited by plans:
\emph{a)}~vertex collisions, i.e., two agents are on the same vertex simultaneously, and
\emph{b)}~swap collisions, i.e., two agents swap their hosting vertices simultaneously.
Note that an agent can remain at its hosting vertex if its executing path contains the same vertex consecutively.

\paragraph{Sequential Execution Model (SEQ)}
In this model, agents take actions sequentially but we cannot control how agents are scheduled at runtime.
More precisely, given an infinite sequence of agents \E called \emph{execution schedule}, the agents are \emph{activated} in turn according to \E.
An activated agent performs the following:
\emph{(i)}~change executing paths if necessary,
\emph{(ii)}~move to its next vertex specified by the progress index if the vertex is unoccupied by other agents, and
\emph{(iii)}~increment its progress index if moved.
The agent remains on its hosting vertex when the next vertex is occupied.
Agents may crash at any time, except for the duration of the procedures for activation.
\E is unknown when offline planning, but we assume that every agent appears infinitely-many times in \E.

\paragraph{Solution}
Given an MAPPCF instance, a \emph{solution for SYN} is a set of plans $\{\plan_1, \plan_2, \ldots, \plan_n\}$ respectively for each agent, such that:
\begin{enumerate}
\item The primary path of $\plan_i$ begins with a start $s_i$.
\item For each path that is not primary, the path begins with a vertex where the agent changes its executing path.
\item The agent $i$ is ensured to reach its goal $g_i$ provided that $i$ follows $\plan_i$, regardless of other agents' crashes, when the total number of crashes is up to $f$.
\end{enumerate}
A \emph{solution for SEQ} is similar to SYN but the third condition should be satisfied for any execution schedule.
\Cref{fig:example} presents solution examples for both models.
Without crash assumptions (i.e., if $f=0$), solutions for SYN are equivalent to those of classical MAPF~\cite{stern2019def}, and solutions for SEQ are equivalent to those of OTIMAPP~\cite{okumura2022offline}.
We assume $f > 0$ in the reminder.

{
  \newcommand{\edgesize}{0.7}
  \newcommand{\boxsize}{2.6}
  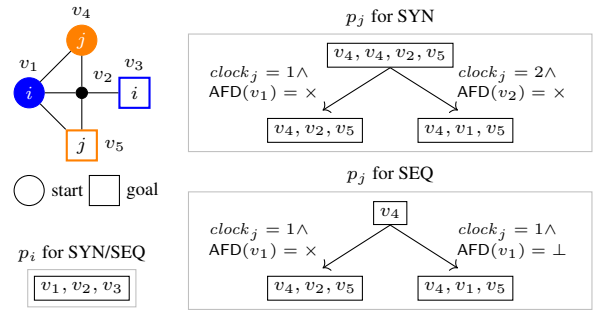
\begin{figure}[tb!]
    \centering
    \scriptsize
    \begin{tikzpicture}
      \begin{scope}[shift={(0, 0)}]
        \node[start-vertex,blue,text=white,label=above:{$v_1$}](v1) at (0, 0) {$i$};
        \node[vertex,label=above right:{$v_2$}](v2) at ($(v1.center)+(\edgesize,0)$) {};
        \node[goal-vertex,draw=blue,thick,label=above:{$v_3$}](v3) at ($(v1.center)+(2*\edgesize,0)$) {$i$};
        \node[start-vertex,orange,text=white,label=above:{$v_4$}](v4) at ($(v1.center)+(\edgesize,\edgesize)$) {$j$};
        \node[goal-vertex,draw=orange,thick,label=right:{$v_5$}](v5) at ($(v1.center)+(\edgesize,-\edgesize)$) {$j$};
        \foreach \u / \v in {v1/v2,v2/v3,v2/v4,v2/v5,v1/v4,v1/v5}
        \draw[line] (\u) -- (\v);
      \end{scope}
      \begin{scope}[shift={(0.7,-2.6)}]
        \node[draw=black](p1-1) at (0, 0) {$v_1, v_2, v_3$};
        \coordinate(c1) at (-0.8,-1.2);
        \coordinate(c2) at (0.8,0.2);
        \node[fit=(p1-1),draw=lightgray] {};
        \node[above=0.1 of p1-1] {$\plan_i$ for SYN/SEQ};
      \end{scope}
      \begin{scope}[shift={(4.8, 0.5)}]
        \node[draw=black](p2-1) at (0, 0) {$v_4, v_4, v_2, v_5$};
        \node[draw=black](p2-2) at (-1, -1) {$v_4, v_2, v_5$};
        \node[draw=black](p2-3) at (1, -1) {$v_4, v_1, v_5$};
        \draw[line,->](p2-1.south) -- ($(p2-2.north)+(0.1,0.1)$);
        \draw[line,->](p2-1.south) -- ($(p2-3.north)+(-0.1,0.1)$);
        \node[anchor=west] at (-2.5,-0.25) {\tiny $\clock{j} = 1 \land$};
        \node[anchor=west] at (-2.5,-0.5) {\tiny $\AFD(v_1)=\times$};
        \node[anchor=west] at (0.8,-0.25) {\tiny $\clock{j} = 2 \land$};
        \node[anchor=west] at (0.8,-0.5) {\tiny $\AFD(v_2)=\times$};
        \coordinate(c1) at (-\boxsize,-1.2);
        \coordinate(c2) at (\boxsize,0.2);
        \node[fit=(c1)(c2),draw=lightgray,label=above:{$\plan_j$ for SYN}] {};
      \end{scope}
      \begin{scope}[shift={(4.8, -1.6)}]
        \node[draw=black](p2-1) at (0, 0) {$v_4$};
        \node[draw=black](p2-2) at (-1, -1) {$v_4, v_2, v_5$};
        \node[draw=black](p2-3) at (1, -1) {$v_4, v_1, v_5$};
        \draw[line,->](p2-1.south) --
        ($(p2-2.north)+(0.1,0.1)$);
        \draw[line,->](p2-1.south) --
        ($(p2-3.north)+(-0.1,0.1)$);
        \node[anchor=west] at (-2.5,-0.25) {\tiny $\clock{j} = 1 \land$};
        \node[anchor=west] at (-2.5,-0.5) {\tiny $\AFD(v_1)=\times$};
        \node[anchor=west] at (0.8,-0.25) {\tiny $\clock{j} = 1 \land$};
        \node[anchor=west] at (0.8,-0.5) {\tiny $\AFD(v_1)=\bot$};
        \coordinate(c1) at (-\boxsize,-1.2);
        \coordinate(c2) at (\boxsize,0.2);
        \node[fit=(c1)(c2),draw=lightgray,label=above:{$\plan_j$ for SEQ}] {};
      \end{scope}
      \begin{scope}[shift={(0, -1.3)}]
        \node[start-vertex, label=right:{\scriptsize start}](label-s) at (0, 0){};
        \node[goal-vertex, label=right:{\scriptsize goal}](label-g) at (1.0, 0){};
      \end{scope}
    \end{tikzpicture}
    \caption{Solution example with AFD.}
    \label{fig:example}
  \end{figure}
}

\paragraph{Remarks}
An index of sequences starts at one.
In SYN, a solution must prevent collisions.
In SEQ, agents are assumed to follow planned paths, while avoiding collisions locally (e.g., by adjusting their velocity).
Rather, a solution must prevent deadlocks, wherein several agents block their progress from each other.
The detailed analyses for SEQ appear in~\cite{okumura2022offline}.
Implementations of failure detectors depend on applications, e.g., using heartbeats as commonly used in distributed network systems~\cite{felber1999failure} or multi-robot platforms such that environments can detect robot faults~\cite{kameyama2021active}.
Herein, failure detectors are assumed to be black-box functions.

\section{Preliminary Analysis}
\label{sec:prelim}

We first present two fundamental analyses to grasp the characteristics of MAPPCF: the model power and the necessary condition for instance to include a solution.

\subsection{Model Power}
A model $X$ is \emph{weakly stronger} than another model $Y$ when all solvable instances in $Y$ are also solvable in $X$.
$X$ is \emph{strictly stronger} than $Y$ when it is weakly stronger than $Y$ and there exists an instance that is solvable in $X$ but unsolvable in $Y$.
Two models are \emph{equivalent} when both are respectively weakly stronger than another.

A \emph{model} of MAPPCF is specified by two components: \emph{(i)} whether the failure detector is anonymous (AFD) or named (NFD), and \emph{(ii)} whether the execution model is synchronous (SYN) or sequential (SEQ).
Characterizing model power, i.e., which model is stronger than another, is important because it can be instrumental when implementing the algorithm.
For instance, AFD is intuitively easier to implement than NFD.
So, if those two models have equivalent power, then we may not need to realize NFD.

The main results are summarized in \cref{fig:model-lattice}.
Several relationships are still open questions, e.g., whether NFD is strictly stronger than AFD in SYN.
In what follows, we present three theorems for the model power analysis.

{
  \begin{figure}[th!]
    \centering
    \begin{tikzpicture}
      \tikzset{
        model/.style={rectangle,draw,black,inner sep=5pt,anchor=center},
      }
      \scriptsize
      \begin{scope}[shift={(0, 0)}]
        \node[model](v1) at (0, 0) {SYN + NFD};
        \node[model](v2) at ($(v1)+(-2.5,-0.5)$) {SYN + AFD};
        \node[model](v3) at ($(v1)+(2.5,-0.5)$) {SEQ + NFD};
        \node[model](v4) at ($(v1)+(0,-0.9)$) {SEQ + AFD};
        \foreach \u / \v in {v3/v1,v4/v3,v4/v2}
        \draw[->,thick] (\u) -- (\v);
        \draw[->,thick,densely dashed] (v2) -- (v1);
      \end{scope}
    \end{tikzpicture}
    \caption{
      Relationship of models.
      `$X \rightarrow Y$' denotes that model $Y$ is strictly stronger than model $X$.
      A dashed arrow means weakly stronger relationship.
    }
    \label{fig:model-lattice}
  \end{figure}
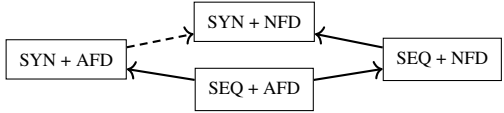
}

\begin{theorem}
  When using the same execution model, NFD is weakly stronger than AFD.
\end{theorem}
\begin{proof}
  NFD can emulate AFD by dropping ``who.''
\end{proof}

\begin{theorem}
  When using the same failure detector types, SYN is strictly stronger than SEQ.
\end{theorem}
\begin{proof}
  \Cref{fig:sync-seq} shows an instance that is solvable for SYN but unsolvable for SEQ.
  In SYN, the agent $j$ can wait until $i$ passes the middle two vertices, and according to crash patterns, $j$ can change its path towards its goal.
  However, in SEQ, there are execution schedules that $j$ enters either of the middle two vertices prior to $i$ because $j$ cannot distinguish whether $i$ is on its start or goal.
  If $j$ is crashed there and $i$ still remains at its start, $i$ cannot reach its goal.

  Next, we prove that every solvable instance in SEQ is solvable in SYN.
  Consider constructing a new solution $Z\sub{syn}$ in SYN given a solution $Z\sub{seq}$ in SEQ.
  This is achieved by, considering one execution schedule (e.g., (1, 2, \ldots, n, 1, 2, \ldots, n, \ldots)) and allowing $Z\sub{syn}$ to move agents in their turn.
  For instance, at timestep one, only agent-1 is allowed to move, and at timestep two, only agent-2 is allowed to move, and so forth.
  With appropriate modifications of paths specified $Z\sub{seq}$, since $Z\sub{seq}$ solves the original instance, $Z\sub{syn}$ also solves the instance in SYN.
\end{proof}

{
  \newcommand{\edgesize}{0.6cm}
  \newcommand{\edgesizeB}{0.3cm}
  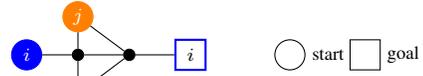
\begin{figure}[tb!]
    \centering
    \scriptsize
    \begin{tikzpicture}
      \begin{scope}[shift={(0, 0)}]
        \node[start-vertex,blue,text=white](v1) at (0, 0) {$i$};
        \node[vertex,right=\edgesize of v1.center](v2) {};
        \node[vertex,right=\edgesize of v2.center](v3) {};
        \node[goal-vertex,draw=blue,thick,right=\edgesize of v3.center](v4) {$i$};
        \node[start-vertex,orange,text=white,above=\edgesizeB of v2.center](v5) {$j$};
        \node[goal-vertex,draw=orange,thick,below=\edgesizeB of v2.center](v6) {$j$};
        \foreach \u / \v in {v1/v2,v2/v3,v3/v4,v2/v5,v3/v5,v2/v6,v3/v6}
        \draw[line] (\u) -- (\v);
      \end{scope}
      \begin{scope}[shift={(3.5, 0)}]
        \node[start-vertex, label=right:{\scriptsize start}](label-s) at (0, 0){};
        \node[goal-vertex, label=right:{\scriptsize goal}](label-g) at (1.0, 0){};
      \end{scope}
    \end{tikzpicture}
    \caption{Solvable instance in SYN but unsolvable in SEQ.}
    \label{fig:sync-seq}
  \end{figure}
}

\begin{theorem}
  In SEQ, NFD is strictly stronger than AFD.
  \label{thrm:nfd-afd}
\end{theorem}
\begin{sketch}
  We show an instance that is only solvable with NFD in SEQ in the Appendix.
\end{sketch}

\subsection{Necessary Condition}
\label{sec:necessary-condition}

\begin{theorem}
  Regardless of execution models and failure detector types, two conditions are necessary for instances to contain solutions.
  \begin{itemize}
  \item No use of other goals: For each agent $i$, there exists a path in $G$ from $s_i$ to $g_i$ that does not include any  $g_j$, for all $j \in A, j \neq i$.
  \item Limitation of other starts: For each agent $i$, for each $B \subset A$ where $i \not\in B$ and $|B| = f$, there exists a path in $G$ from $s_i$ to $g_i$ that does not include $s_j$, for all $j \in B$.
  \end{itemize}
  \label{thrm:necessary}
\end{theorem}
\begin{proof}
  \emph{No use of other goals}:
  Suppose that there exists an agent $i$ that needs to pass through one of the goals $g_j$. If $i$ crashes at $g_j$, then $j$ cannot reach $g_j$.
  \emph{Limitation of other starts}:
  Suppose that there exists an agent $i$ that needs to pass through one of the starts of $B$. If all agents in $B$ are crashed at their starts, then $i$ cannot reach $g_i$.
\end{proof}

When $f = |A|-1$, the conditions in \cref{thrm:necessary} are equivalent to a \emph{well-formed instance}~\cite{vcap2015prioritized} for MAPF; an instance such that every agent has at least one path that uses no others' start and goal vertices.
Our condition slightly differs in the limitation of starts because agents can change their behavior at runtime, according to failure detectors.

\section{Computational Complexity}
\label{sec:complexity}

This section discusses the complexity of MAPPCF, namely, studying two questions: the difficulty of finding solutions and that of verifying solutions.
The primary result is that both problems are computationally intractable; the former is NP-hard and the latter is co-NP-complete.
Both proofs are based on reductions of the 3-SAT problem, determining the satisfiability of a formula in conjunctive normal form with three literals in each clause.
The proof of verification appears in the Appendix.

{
  \newcommand{\edgesizeA}{1.3}  
  \newcommand{\edgesizeB}{0.8}  
  \newcommand{\edgesizeC}{1.3}  
  \newcommand{\edgesizeD}{0.5}  
  \newcommand{\edgesizeE}{1.4}  
  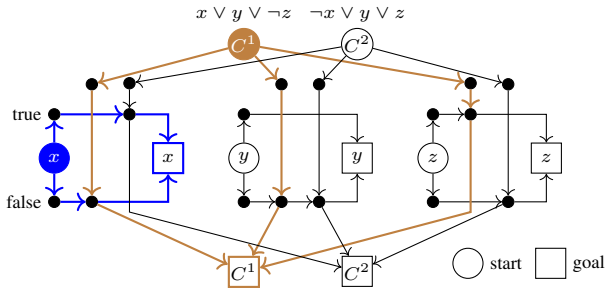
\begin{figure}[tb!]
    \centering
    \scriptsize
    \begin{tikzpicture}
      \begin{scope}[shift={(0, 0)}]
        %
        \node[start-vertex,blue,text=white](v1-s) at (0, 0) {$x$};
        \node[goal-vertex,right=\edgesizeA of v1-s.center,draw=blue,thick](v1-g) {$x$};
        \node[start-vertex,right=\edgesizeB of v1-g.center](v2-s) {$y$};
        \node[goal-vertex,right=\edgesizeA of v2-s.center](v2-g) {$y$};
        \node[start-vertex,right=\edgesizeB of v2-g.center](v3-s) {$z$};
        \node[goal-vertex,right=\edgesizeA of v3-s.center](v3-g) {$z$};
        %
        \node[start-vertex,above=\edgesizeC of v2-s.center,label=above:{$x \lor y \lor \lnot z$},brown,text=white]
        (c1-s) {$C^1$};
        \node[start-vertex,above=\edgesizeC of v2-g.center,label=above:{$\lnot x \lor y \lor z$}]
        (c2-s) {$C^2$};
        \node[goal-vertex,below=\edgesizeC of v2-s.center,draw=brown,thick](c1-g) {$C^1$};
        \node[goal-vertex,below=\edgesizeC of v2-g.center](c2-g) {$C^2$};
        %
        \node[vertex,above=\edgesizeD of v1-s.center,label=left:{true}](v1-u1) {};
        \node[vertex,below=\edgesizeD of v1-s.center,label=left:{false}](v1-b1) {};
        \node[vertex,above=\edgesizeD of v2-s.center](v2-u1) {};
        \node[vertex,below=\edgesizeD of v2-s.center](v2-b1) {};
        \node[vertex,above=\edgesizeD of v3-s.center](v3-u1) {};
        \node[vertex,below=\edgesizeD of v3-s.center](v3-b1) {};
        \node[vertex](v1-b2) at ($(v1-b1)+(0.5,0)$) {};
        \node[vertex](v1-u2) at ($(v1-u1)+(1.0,0)$) {};
        \node[vertex](v2-b2) at ($(v2-b1)+(0.5,0)$) {};
        \node[vertex](v2-b3) at ($(v2-b1)+(1.0,0)$) {};
        \node[vertex](v3-u2) at ($(v3-u1)+(0.5,0)$) {};
        \node[vertex](v3-b2) at ($(v3-b1)+(1.0,0)$) {};
        %
        \node[vertex,above=\edgesizeE of v1-b2](c1-1) {};
        \node[vertex,above=\edgesizeE of v2-b2](c1-2) {};
        \node[vertex,above=0.25 of v1-u2](c2-1) {};
        \node[vertex,above=\edgesizeE of v2-b3](c2-2) {};
        \node[vertex,above=0.25 of v3-u2](c3-1) {};
        \node[vertex,above=\edgesizeE of v3-b2](c3-2) {};
        %
        \foreach \u / \v in {v2-s/v2-u1,v2-s/v2-b1,v2-b1/v2-b2,v2-b2/v2-b3,
        v3-s/v3-u1,v3-s/v3-b1,v3-u1/v3-u2,v3-b1/v3-b2}
        \draw[line,->](\u)--(\v);
        \foreach \u / \v in {v1-s/v1-u1,v1-s/v1-b1,v1-b1/v1-b2,v1-u1/v1-u2}
        \draw[line,->,blue,thick](\u)--(\v);
        \foreach \u / \v in {v2-u1/v2-g,v2-b3/v2-g,v3-u2/v3-g,v3-b2/v3-g} \draw[line,->](\u)-|(\v);
        \foreach \u / \v in {v1-u2/v1-g,v1-b2/v1-g} \draw[line,->,blue,thick](\u)-|(\v);
        %
        \foreach \u / \v in {
          c1-s/c1-1,c1-1/v1-b2,v1-b2/c1-g,
          c1-s/c1-2,c1-2/v2-b2,v2-b2/c1-g,
          c1-s/c3-1,c3-1/v3-u2}
        \draw[line,->,brown,thick](\u)--(\v);
        \draw[line,->,brown,thick](v3-u2)--($(v3-u2)+(0,-1.3)$)--(c1-g);
        %
        \foreach \u / \v in {
          c2-s/c2-1,c2-1/v1-u2,
          c2-s/c2-2,c2-2/v2-b3,v2-b3/c2-g,
          c2-s/c3-2,c3-2/v3-b2,v3-b2/c2-g}
        \draw[line,->](\u)--(\v);
        \draw[line,->](v1-u2)--($(v1-u2)+(0,-1.3)$)--(c2-g);
      \end{scope}
      \begin{scope}[shift={(5.5, -1.4)}]
        \node[start-vertex, label=right:{start}](label-s) at (0, 0){};
        \node[goal-vertex, label=right:{goal}](label-g) at (1.1, 0){};
      \end{scope}
    \end{tikzpicture}
    \caption{
      MAPPCF instance on a directed graph in SEQ reduced from the SAT instance $(x \lor y \lor \lnot z) \land (\lnot x \lor y \lor z)$.
    }
    \label{fig:np-hard-directed}
  \end{figure}
}
{
  \newcommand{\edgesizeA}{0.6}
  \newcommand{\edgesizeB}{0.8}  
  \newcommand{\edgesizeC}{0.3}  
  \newcommand{\edgesizeD}{1.5}
  \newcommand{\edgesizeE}{0.7}  
  \newcommand{\edgesizeF}{0.4}  
  \begin{figure}[tb!]
    \centering
    \scriptsize
    \begin{tikzpicture}
      \begin{scope}[shift={(0, 0)}]
        \node[start-vertex](v1) at (0, 0) {$i$};
        \node[vertex,right=\edgesizeA of v1.center](v2) {};
        \node[vertex,right=\edgesizeA of v2.center,label=below right:{$u$}](v3) {};
        \node[start-vertex,above=\edgesizeF of v2.center,teal,text=white](v4) {$\alpha$};
        \node[goal-vertex,below=\edgesizeF of v2.center,draw=teal,thick](v5) {$\alpha$};
        \foreach \u / \v in {v1/v2,v2/v3,v2/v4,v2/v5,v1/v4,v1/v5}
        \draw[line] (\u) -- (\v);
        \draw[line,densely dotted,thick](v3)--($(v3)+(0.4,0)$);
        \node[] at ($(v5.center)+(0,-0.5)$) {(b)~SEQ};
      \end{scope}
      \begin{scope}[shift={(2.3, 0)}]
        \node[start-vertex](v1) at (0, 0) {$i$};
        \node[vertex,right=\edgesizeA of v1.center](v2) {};
        \node[goal-vertex,above=\edgesizeF of v2.center,thick,draw=blue](v3) {$\beta$};
        \node[goal-vertex,below=\edgesizeF of v2.center,thick,draw=orange](v4) {$\gamma$};
        \node[start-vertex,orange,text=white](v5) at ($(v3.center)+(\edgesizeB,-\edgesizeC)$) {$\gamma$};
        \node[start-vertex,blue,text=white](v6) at ($(v4.center)+(\edgesizeB,\edgesizeC)$) {$\beta$};
        \node[vertex,right=\edgesizeD of v2.center,label=below right:{$u$}](v7) {};
        \foreach \u / \v in {v1/v2,v2/v3,v2/v4,v3/v5,v4/v6,v2/v5,v2/v6,v5/v6,v5/v7,v6/v7}
        \draw[line] (\u) -- (\v);
        \foreach \u / \v in {v3/v7,v4/v7}
        \draw[line] (\u) -| (\v);
        \draw[line,densely dotted,thick](v7)--($(v7)+(0.4,0)$);
        \node[] at ($(v4.center)+(0.7,-0.5)$) {(c)~SYN};
      \end{scope}
      {
        \begin{scope}[shift={(-1.8, 0)}]
          \node[start-vertex](v1) at (0, 0) {$x$};
          \node[vertex,above=\edgesizeE of v1.center](v2) {};
          \node[vertex,below=\edgesizeE of v1.center](v3) {};
          \foreach \u / \v in {v1/v2,v1/v3} \draw[line] (\u) -- (\v);
          %
          \draw[line,fill=white]
          ($(v1)+(-0.2,0.3)$)--
          ($(v1)+(0.2,0.3)$)--
          ($(v1)+(0,0.6)$)--cycle;
          \draw[line]($(v1)+(-0.2,0.6)$)--($(v1)+(0.2,0.6)$);
          \draw[line,fill=white]
          ($(v1)+(-0.2,-0.3)$)--
          ($(v1)+(0.2,-0.3)$)--
          ($(v1)+(0,-0.6)$)--cycle;
          \draw[line]($(v1)+(-0.2,-0.6)$)--($(v1)+(0.2,-0.6)$);
          \draw[line](v2)--($(v2)+(0.6,0)$);
          \draw[line](v3)--($(v3)+(0.6,0)$);
          \node[rotate=90] at ($(v2)+(0.7,0)$) {$\sim$};
          \node[rotate=90] at ($(v2)+(0.78,0)$) {$\sim$};
          \node[rotate=90] at ($(v3)+(0.7,0)$) {$\sim$};
          \node[rotate=90] at ($(v3)+(0.78,0)$) {$\sim$};
          \node[anchor=west] at (0.2, 0.50) {diode};
          \node[anchor=west] at (0.2, 0.30) {gadget};
          \node[] at (0.3,-1.12) {(a) usage};
        \end{scope}
      }
    \end{tikzpicture}
    \caption{Diode gadgets.}
    \label{fig:diode}
  \end{figure}
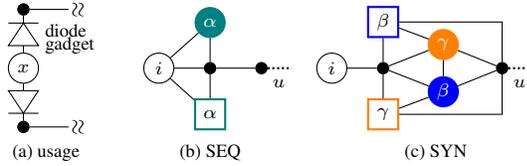
}

\begin{theorem}
  MAPPCF is NP-hard regardless of models.
\end{theorem}
\begin{proof}
  We first prove that MAPPCF on digraphs in SEQ is NP-hard.
  The proof is done by reduction of the SAT problem and works regardless of failure detector types.
  Throughout the proof, we use the following example:
  $(x \lor y \lor \lnot z) \land (\lnot x \lor y \lor z)$.
  The reduction is depicted in \cref{fig:np-hard-directed}.

  \smallskip
  \noindent
  \emph{A. Construction of an MAPPCF Instance:}
  We first introduce a \emph{variable agent} for each variable $x_i$.
  In \cref{fig:np-hard-directed}, we highlight the corresponding agent of the variable $x$ as blue-colored.
  The reduced instance has two paths for each variable agent: \emph{upper} or \emph{lower} paths.
  Both paths include at least one vertex (just above/below of the start in \cref{fig:np-hard-directed}) and additional vertices depending on clauses of the formula.

  We next introduce a \emph{clause} agent for each clause $C^j$ of the formula.
  Each clause agent has multiple paths, corresponding to each literal in the clause.
  Those paths contain two vertices excluding the start and the goal: one vertex unique to each literal, and another vertex shared with the corresponding variable agent.
  The shared vertex is located on the lower (or upper) path of the variable agent when the literal is positive (resp. negative).
  In \cref{fig:np-hard-directed}, we highlight the corresponding agent of the clause $C^1 =  x \lor y \lor \lnot z$ as brown-colored.
  The translation from the formula into an MAPPCF instance is clearly done in polynomial time.

  \smallskip
  \noindent
  \emph{B. MAPPCF has a solution if the formula is satisfiable:}
  Given a satisfiable assignment, a solution of MAPPCF is built as follows.
  When a variable is assigned true (or false), let the corresponding agent takes the upper (resp. lower) path.
  Each clause agent then has at least one path (vertex) disjoint with any variable agent; otherwise, the clause is unsatisfied.
  Let the clause agent take this path;
  these paths constitute a solution because all paths are disjoint.

  \smallskip
  \noindent
  \emph{C. The formula is satisfiable if MAPPCF has a solution:}
  In every solution, a plan for each agent inherently consists of a single path, due to the instance construction.
  These paths should be (vertex) disjoint; otherwise, the crash of one agent blocks another from reaching its goal.
  Then, build an assignment as follows.
  Assign a variable true (or false) when the corresponding variable agent uses the upper (resp. lower) path.
  This assignment is satisfiable because it ensures at least one literal is satisfied in all clauses.

  \smallskip
  \noindent
  \emph{D. Extending the reduction to undirected graphs:}
  The aforementioned proof is extended to the undirected case by introducing a \emph{diode} gadget to the starts of every agent, as partially shown in \cref{fig:diode}a.
  This gadget prevents back to the start once the agent passes through the gadget (i.e., reaching vertex $u$ in \cref{fig:diode}b--c).
  Therefore, the same proof procedure (i.e., finding disjoint paths) is applied to other models.
  The gadget for SEQ and SYN are shown in \cref{fig:diode}b--c, respectively.
  The proofs of their properties are delivered in the Appendix.
\end{proof}

\begin{theorem}
  Verifying a solution of MAPPCF is co-NP-complete regardless of models.
  \label{thrm:verification}
\end{theorem}

Although MAPPCF is intractable in general, it is notable that the difficulties might be relaxed in certain classes, e.g., instances on planner graphs or when $f=1$.

\section{Solving MAPPCF}
\label{sec:solver}
We next discuss how to solve MAPPCF.
The main challenge is how to manage crash awareness differences among agents.
For instance, agent $i$ may observe a crash of agent $j$ while at a neighboring position, and change its path accordingly.
However, another correct agent $k$, located further away from $j$, might not be aware that $j$ has crashed.
To preserve safety, a plan of $k$ requires avoiding collisions and deadlocks with both before-after paths of $i$ (that is, regardless of the crash of $j$).
Any planning algorithm solving MAPPCF thus has to cope with different awareness of crash patterns.

The proposed method, \emph{decoupled crash faults resolution framework (\algoname)}, returns an MAPPCF solution.
Herein, we consider only NFD, but \algoname is also applicable to AFD.

\subsection{Framework Description}

\Cref{algo:planner} presents \algoname.
\Cref{fig:planner-viz} illustrates a running example in SYN.
We first describe \algoname in SYN using this example, followed by detailed parts of the implementation.

{
  \newcommand{\U}{\m{\mathcal{U}}}
  \newcommand{\I}{\m{\mathcal{I}}}
  \renewcommand{\P}{\m{\mathcal{P}}}
  \renewcommand{\E}{\m{\mathcal{E}}}
  \begin{algorithm}[t!]
    \caption{\algoname}
    \label{algo:planner}
    \begin{algorithmic}[1]
    \item[\textbf{input}:~instance \I;\;\;\textbf{output}:~solution \P or \FAILURE]
      \State $\P \leftarrow \funcname{get\_initial\_plans(\I)}$
      \label{algo:planner:init}
      \State $\U \leftarrow \funcname{get\_initial\_unresolve\_events}(\I, \P)$
      \label{algo:planner:init-unresolved}
      \While{$\U \neq \emptyset$}
      \label{algo:planner:while}
      \State $e \leftarrow \U.\funcname{pop}()$
      \Comment{event; pair of crash \& effect}
      \State $\pi \leftarrow \funcname{find\_backup\_path}(\I, \P, e)$
      \label{algo:planner:backup}
      \IFSINGLE{$\pi$ not found}{\Return \FAILURE}
      \State $\U.\funcname{push}(\funcname{get\_new\_unresolved\_events}(\I, \P, \pi))$
      \State update $\P$ with $\pi$
      \EndWhile
      \label{algo:planner:end-while}
      \State \Return \P
      \label{algo:planner:return}
    \end{algorithmic}
  \end{algorithm}
}

\paragraph{Finding Initial Plans}
DCRF first obtains an initial plan for each agent (i.e., a path) [\cref{algo:planner:init}].
This process is equivalent to solving MAPF with existing solvers.
In the example, the initial plans for agents $i$, $j$, and $k$ are respectively depicted in \cref{fig:planner-viz}d, \ref{fig:planner-viz}b, and \ref{fig:planner-viz}c.

\paragraph{Identifying Unresolved Events}
DCRF next identifies \emph{unresolved events}.
An event is a pair of \emph{crash} (i.e., who crashes where and when), and \emph{effect} (i.e., whose path is affected where and when).
Finding unresolved events is done by finding shared vertices in a set of paths.
In the example (\cref{fig:planner-viz}d), the initial plans contain two unresolved events:
\begin{itemize}
\item $e^1$: $\P_i$ cannot use $\langle v_2; t=2\rangle$ if $j$ is crashed at $\langle v_2; t=1\rangle$
\item $e^2$: $\P_i$ cannot use $\langle v_3; t=3\rangle$ if $k$ is crashed at $\langle v_3; t=2\rangle$
\end{itemize}
These events should be resolved by preparing \emph{backup paths}.
\algoname resolves events in a decoupled manner as follows.

{
  \newcommand{\edgesizeA}{0.8}
  \newcommand{\drawG}{
    \node[vertex](v1) at (0, 0) {};
    \node[vertex](v2) at ($(v1.center)+(\edgesizeA,0)$) {};
    \node[vertex](v3) at ($(v1.center)+(2*\edgesizeA,0)$) {};
    \node[vertex](v4) at ($(v1.center)+(3*\edgesizeA,0)$) {};
    \node[vertex](v5) at ($(v1.center)+(\edgesizeA,-\edgesizeA)$) {};
    \node[vertex](v6) at ($(v1.center)+(2*\edgesizeA,-\edgesizeA)$) {};
    \node[vertex](v7) at ($(v1.center)+(2*\edgesizeA,\edgesizeA)$) {};
    \foreach \u / \v in {v1/v2,v2/v3,v3/v4,v2/v5,v3/v6,v5/v6,v1/v5,v6/v4,v7/v3,v7/v4,v2/v7,v5/v3}
    \draw[line] (\u) -- (\v);
  }
  \newcommand{\cross}{$\mathbin{\tikz [x=1.4ex,y=1.4ex,line width=.2ex, red] \draw (0,0) -- (1,1) (0,1) -- (1,0);}$}
  \begin{figure}[tb!]
    \centering
    \small
    \begin{tikzpicture}
      \begin{scope}[shift={(0, 0)}]
        \drawG
        \node[start-vertex,blue,text=white,label=above:{$v_1$}] at (v1) {$i$};
        \node[goal-vertex,draw=blue,thick,label=above:{$v_4$}] at (v4) {$i$};
        \node[start-vertex,orange,text=white,label=above:{$v_2$}] at (v2) {$j$};
        \node[goal-vertex,draw=orange,thick,label=below:{$v_5$}] at (v5) {$j$};
        \node[start-vertex,teal,text=white,label=above:{$v_7$}] at (v7) {$k$};
        \node[goal-vertex,draw=teal,thick,label=below:{$v_6$}] at (v6) {$k$};
        \node[above right=-0.1 of v3] {$v_3$};
        \node[above=0.5 of v1] {(a)};
      \end{scope}
      \begin{scope}[shift={(0, -2.8)}]
        \drawG
        \draw[line,->,blue,thick, densely dotted]($(v1)+(0.05,-0.05)$)--($(v4)+(-0.1,-0.05)$);
        \draw[line,->,teal,thick, densely dotted]($(v7)+(0.05,-0.05)$)--($(v6)+(0.05,0.1)$);
        \node[start-vertex,orange,text=white](j-s) at (v2) {$j$};
        \node[goal-vertex,draw=orange,thick](j-g) at (v5) {$j$};
        \draw[line,->,very thick,orange](j-s)--(j-g);
        \node[fit=(v1)(v4)(v7)(v5), draw=lightgray,inner sep=4pt,label=above:{$\P_j$}] {};
        \node[above=0.5 of v1] {(b)};
      \end{scope}
      \begin{scope}[shift={(0, -5.6)}]
        \drawG
        \draw[line,->,blue,thick, densely dotted]($(v1)+(0.05,-0.05)$)--($(v4)+(-0.1,-0.05)$);
        \draw[line,->,orange,thick, densely dotted]($(v2)+(0.05,-0.05)$)--($(v5)+(0.05,0.1)$);
        \node[start-vertex,teal,text=white](k-s) at (v7) {$k$};
        \node[goal-vertex,draw=teal,thick](k-g) at (v6) {$k$};
        \draw[line,->,very thick,teal](k-s)--(k-g);
        \node[fit=(v1)(v4)(v7)(v5), draw=lightgray,inner sep=4pt,label=above:{$\P_k$}] {};
        \node[above=0.5 of v1] {(c)};
      \end{scope}
      \begin{scope}[shift={(4.2, 0)}]
        \drawG
        \node[start-vertex,blue,text=white](i-s) at (v1) {$i$};
        \node[goal-vertex,draw=blue,thick](i-g) at (v4) {$i$};
        \draw[line,->,very thick,blue](i-s)--(i-g);
        \draw[line,->,orange,thick, densely dotted]($(v2)+(0.05,-0.05)$)--($(v5)+(0.05,0.1)$);
        \draw[line,->,teal,thick, densely dotted]($(v7)+(0.05,-0.05)$)--($(v6)+(0.05,0.1)$);
        \node[above=0.5 of v1] {(d)};
      \end{scope}
      \begin{scope}[shift={(3.2, -2.8)}]
        \drawG
        \draw[line,->,teal,thick, densely dotted]($(v7)+(0.05,-0.05)$)--($(v6)+(0.05,0.1)$);
        \node[start-vertex,blue,text=white](i-s) at (v1) {$i$};
        \node[goal-vertex,draw=blue,thick](i-g) at (v4) {$i$};
        \node[start-vertex,orange,text=white](j-s) at (v2) {$j$};
        \node[above left=-0.2 of j-s] {\cross};
        \draw[line,->,very thick,blue](i-s)--(v5)--(v3)--(i-g);
        \node[above=0.5 of v1] {(e)};
      \end{scope}
      \begin{scope}[shift={(5.4, -4.0)}]
        \drawG
        \draw[line,->,orange,thick, densely dotted]($(v2)+(0.05,-0.05)$)--($(v5)+(0.05,0.1)$);
        \node[start-vertex,blue,text=white](i-s) at (v2) {$i$};
        \node[goal-vertex,draw=blue,thick](i-g) at (v4) {$i$};
        \node[start-vertex,teal,text=white](k-s) at (v3) {$k$};
        \node[above left=-0.2 of k-s] {\cross};
        \draw[line,->,very thick,blue](i-s)--(v7)--(i-g);
        \node[above=0.1 of v1] {(g)};
      \end{scope}
      \begin{scope}[shift={(3.2, -5.6)}]
        \drawG
        \node[start-vertex,blue,text=white](i-s) at (v5) {$i$};
        \node[goal-vertex,draw=blue,thick](i-g) at (v4) {$i$};
        \node[start-vertex,orange,text=white](j-s) at (v2) {$j$};
        \node[above left=-0.2 of j-s] {\cross};
        \node[start-vertex,teal,text=white](k-s) at (v3) {$k$};
        \node[above left=-0.2 of k-s] {\cross};
        \draw[line,->,very thick,blue](i-s)--(v6)--(i-g);
        \node[above=0.5 of v1] {(f)};
      \end{scope}
      {
        \coordinate[](c1) at (5.3,-1.0);
        \draw[line,->](c1)--(4.9,-1.8);
        \node[anchor=west] at (2.9,-0.9) {\tiny event: $e^1$};
        \node[anchor=west] at (2.9,-1.15) {\tiny crash: $\langle j; v_2; t=1\rangle$};
        \node[anchor=west] at (2.9,-1.4) {\tiny effect: $\langle v_2; t=2\rangle$};
        \draw[line,->](c1)--(6.5,-3.0);
        \node[anchor=west] at (5.7,-1.1) {\tiny event: $e^2$};
        \node[anchor=west] at (5.7,-1.35) {\tiny crash: $\langle k; v_3; t=2 \rangle$};
        \node[anchor=west] at (5.7,-1.6) {\tiny effect: $\langle v_3; t=3 \rangle$};
        \draw[line,->](4.9,-3.8)--(4.9,-4.6);
        \node[anchor=west] at (2.8,-3.9)  {\tiny event: $e^3$};
        \node[anchor=west] at (2.8,-4.15) {\tiny crash: $\langle k; v_3; t=2 \rangle$};
        \node[anchor=west] at (2.8,-4.40) {\tiny effect: $\langle v_3; t=3 \rangle$};
        \coordinate[](c2) at (2.85,0.9);
        \coordinate[](c3) at (8.0,-6.53);
        \node[fit=(c2)(c3),draw=lightgray,label=above:{$\P_i$}] {};
      }
    \end{tikzpicture}
    \caption{
      Running example of \algoname in SYN.
      A red cross corresponds to an observed crashed agent.
      Dotted lines are paths with which the planning should avoid collisions.
    }
    \label{fig:planner-viz}
  \end{figure}
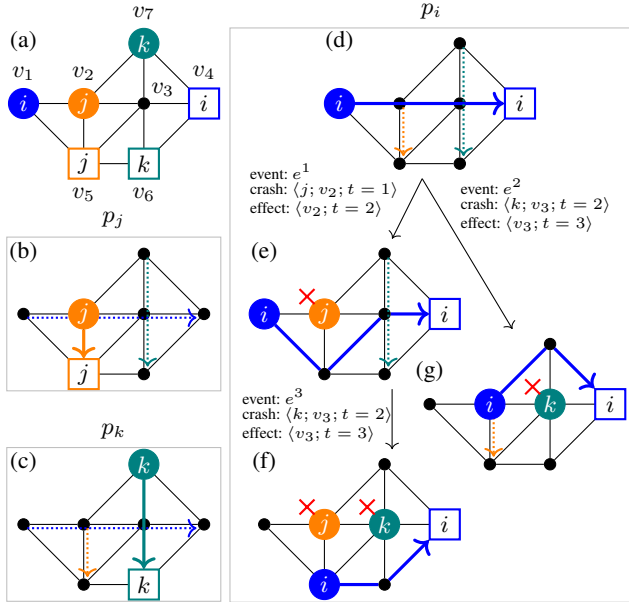
}

\paragraph{Resolving Events}
Unresolved events are stored in a priority queue $\mathcal{U}$ and handled one by one [Lines~\ref{algo:planner:while}--\ref{algo:planner:end-while}].
For each event, DCRF tries to find a backup path [\cref{algo:planner:backup}].
This is a single-agent pathfinding problem, whose goal location is the same as the initial path, while the start location is one-step before where the crashed agent is.
The pathfinding is constrained to avoid collisions with observed crashed agents and other already constructed plans.
If failed to find such a path, \algoname reports \FAILURE.
Otherwise, \algoname identifies new unresolved events with the backup path, and updates the plan.
The event is now resolved.
When all events are resolved, the framework returns a solution [\cref{algo:planner:return}].
In the example, when the event $e^1$ is popped from $\mathcal{U}$, \algoname computes the backup path shown in \cref{fig:planner-viz}e.
Observe that this backup path must not use $v_6$ to avoid collisions with $p_k$.
\algoname then identifies and registers a new unresolved event $e^3$, and updates $i$'s plan.
\algoname continues applying the same procedure for the events $e^2$ (\cref{fig:planner-viz}g) and $e^3$ (\cref{fig:planner-viz}f).

\paragraph{Remark}
\algoname is incomplete, i.e., it does not guarantee to return a solution even though an instance is solvable.
Such a failure example is available in the Appendix.

\subsection{Implementation Details}
\label{sec:impl}

\paragraph{Discarding Unnecessary Events}
When one crash affects a given path several times, only the first effect should be resolved.
This happens when the path is not simple.
For this purpose, the queue $\mathcal{U}$ stores events in ascending order of their occurring time.
DCRF then discards events when encountering already resolved crashes.

\paragraph{Inconsistent Crash Patterns}
When preparing a backup path for agent $i$, pathfinding does not necessarily need to avoid collisions with all others' paths.
For instance, \cref{fig:planner-viz}e assumes $j$ has crashed.
Therefore, in descendant backup paths for $i$, $i$ can neglect $j$'s plan.
Similarly, when two different paths assume crashes at distinct locations of the same agent, or when two paths assume more than $f$ crashes in total, those two paths cannot be executed during the same execution.
Consequently, \algoname does not need to identify unresolved events between these paths.

\paragraph{Refinement of Initial Paths}
Reducing the number of events is crucial for constructing solutions, as a higher number of events yields a higher number of paths for each agent.
Consequently, when preparing a backup path, the pathfinding process tries to avoid collisions with those many paths, and may fail to find a suitable path.
To circumvent this problem, when preparing initial plans, we introduce a refinement phase that minimizes the use of shared vertices between agents.
This is done by adapting the technique to improve MAPF solutions~\cite{okumura2021iterative}.

\paragraph{Implementation for SEQ}
\algoname is applicable to SEQ by simply replacing `collision' and `MAPF' by `deadlocks' and `OTIMAPP,' respectively.

\paragraph{Implementation for AFD}
For AFD, the same workflow is available by assuming that observed crashes are anonymous.

{
  \newcommand{\colsize}{0.15\linewidth}
  \setlength{\tabcolsep}{1pt}
  \newcommand{\block}[3]{
    &
    \begin{minipage}{\colsize}
      \centering
      \includegraphics[width=1.0\linewidth]{fig/raw/failure_#1_#2_fix_#3.pdf}
    \end{minipage}
    &
    \begin{minipage}{\colsize}
      \centering
      \includegraphics[width=1.0\linewidth]{fig/raw/runtime_#1_#2_fix_#3.pdf}
    \end{minipage}
    &
    \begin{minipage}{\colsize}
      \centering
      \includegraphics[width=1.0\linewidth]{fig/raw/cost_#1_#2_fix_#3.pdf}
    \end{minipage}
  }
  \begin{figure*}[th!]
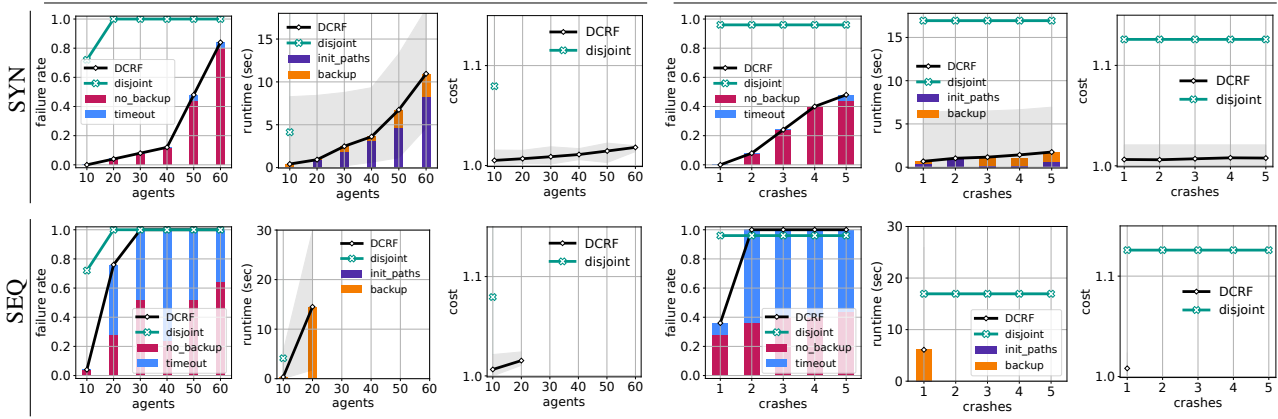

    \centering
    \small
    \begin{tabular}{rcccccc}
      \multicolumn{7}{c}{\mapname{random-32-32-10} (32x32; $|V|=922$)}
      \\
      & \multicolumn{3}{c}{fixed crashes: $f=1$}
      & \multicolumn{3}{c}{fixed agents: $|A|=15$}
      \\\cmidrule(lr){2-4}\cmidrule(lr){5-7}
      \multicolumn{1}{c|}{\rotatebox{90}{SYN}}
      \block{random-32-32-10}{sync}{crash}
      \block{random-32-32-10}{sync}{agent}
      \medskip\\
      \multicolumn{1}{c|}{\rotatebox{90}{SEQ}}
      \block{random-32-32-10}{seq}{crash}
      \block{random-32-32-10}{seq}{agent}
      \smallskip\\
      \multicolumn{7}{c}{\mapname{random-64-64-10} (64x64; $|V|=3,687$)}
      \\
      & \multicolumn{3}{c}{fixed crashes: $f=1$}
      & \multicolumn{3}{c}{fixed agents: $|A|=15$}
      \\\cmidrule(lr){2-4}\cmidrule(lr){5-7}
      \multicolumn{1}{c|}{\rotatebox{90}{SYN}}
      \block{random-64-64-10}{sync}{crash}
      \block{random-64-64-10}{sync}{agent}
      \medskip\\
      \multicolumn{1}{c|}{\rotatebox{90}{SEQ}}
      \block{random-64-64-10}{seq}{crash}
      \block{random-64-64-10}{seq}{agent}
    \end{tabular}
    \caption{
      Summary of results.
      Three types of figures are included.
      \emph{failure rate}:
      We also show failure reasons of \algoname by stacking graphs.
      `no\_backup' means that \algoname failed to prepare a backup path.
      `timeout' means that \algoname reaches the time limit.
      `init\_paths' means that \algoname fails to prepare initial paths.
      \emph{runtime}:
      The average runtime of successful instances is presented, accompanied by minimum and maximum values shown by transparent regions.
      We also show runtime profiling, which is categorized into preparing initial paths (`init\_paths') and computing backup paths (`backup').
      \emph{cost}:
      This rates solution quality when crashes do not happen.
      See the result description for details.
      The average, minimum, and maximum values of the successful instances are shown.
      Note that finding disjoint paths are irrelevant from $f$.
    }
    \label{fig:result-main}
  \end{figure*}
}

\section{Evaluation}
\label{sec:evaluation}

This section evaluates \algoname in both SYN and SEQ with NFD.
We present a variety of aspects including merits to consider MAPPCF and bottlenecks of the planning.

\paragraph{Baseline}
We compared \algoname with a procedure to obtain pairwise vertex-disjoint paths.
The rationale is that disjoint paths are trivially fault-tolerant;
regardless of crash patterns, correct agents can always reach their destinations.
On the other hand, with more agents, it is expected that finding such paths becomes impossible.
The disjoint paths were obtained by an adapted version of conflict-based search~\cite{sharon2015conflict}, a celebrated MAPF algorithm.
The adapted one is complete, i.e., eventually returning disjoint paths if they exist, otherwise, reporting non-existence.

\paragraph{Experimental Design}
MAPPCF has two critical factors: the number of agents $|A|$ and crashes $f$.
To investigate those effects on the planning, we prepared two scenarios: \emph{(i)}~fixing $f$ while changing $|A|$, or \emph{(ii)} fixing $|A|$ while changing $f$.
Each scenario was tested on two four-connected grids (size: 32x32 and 64x64) obtained from~\cite{stern2019def}.
These grids contain randomly placed obstacles (10\%).
For each scenario, grid, and $|A|$, we prepared $25$ well-formed instances~\cite{vcap2015prioritized}, considering the necessary condition for solutions to exist (see~\cref{sec:necessary-condition}); however, unsolvable instances might still be included because it is not sufficient.
The identical instances were used in both SYN and SEQ.

\paragraph{Planning Failure}
We regard that a method succeeds in solving an instance when it returns a solution before the timeout of \SI{30}{\second}; otherwise, the attempt is a failure.

\paragraph{Implementation of \algoname}
The initial paths were obtained by prioritized planning~\cite{vcap2015prioritized,okumura2022offline}, respectively for SYC and SEQ.
Single-agent pathfinding was implemented by \astar, adding a heuristic that penalizes the use of common vertices with other agents' paths.
We informally observed that this improves the success rate due to a smaller number of events.
We applied the refinement over initial paths (\cref{sec:impl}) for SYN but not for SEQ because the effect was subtle (see the Appendix).

\paragraph{Evaluation Environment}
The code was written in Julia and is available in the supplementary material.
The experiments were run on a desktop PC with Intel Core i9-7960X \SI{2.8}{\giga\hertz} CPU and \SI{64}{\giga\byte} RAM.
We executed 32 different instances in parallel using multi-threading.

\paragraph{Results}
\Cref{fig:result-main} shows the results.
The main findings are:

\myitem{Regardless of models, finding solutions become difficult to compute as the number of agents $|A|$ or crashes $f$ increase.}
With larger $|A|$ or $f$, \algoname needs to manage a huge number of crash patterns.
Consequently, \algoname often fails to find backup paths, or reaches timeout.

\myitem{MAPPCF can address more crash situations compared to just finding disjoint paths.}
Note however that the gaps in the success rate between \algoname and finding disjoint paths becomes smaller in SEQ.
This is partially due to finding deadlock-free paths as they are difficult to compute in SEQ.

\myitem{Without crashes, \algoname provides solutions with smaller costs, compared to disjoint paths.}
In \cref{fig:result-main}, we also present the cost of paths that agents are to follow if there is no crash.
For SYN, a cost is total traveling time (aka. sum-of-costs).
For SEQ, a cost is sum of path distance.
Both scores were normalized by sum of distances between start-goal pairs; hence the minimum is one.
Note that the cost is identical with different $f$ if $|A|$ and start-goal locations are the same.
The result indicates that \algoname provides better planning that suppresses redundant agents' motions when the entire system operates correctly.

\smallskip
The success rate of the planning in larger grids with fixed $f=1$ is available in the Appendix.
Since this is the first study of MAPPCF where agents may observe different information at runtime, we acknowledge the room for algorithmic improvements.
Therefore, we consider that further improvements of \algoname or developing new MAPPCF algorithms are promising directions.
Specifically, as seen in \cref{fig:result-main}, a large amount of failure reasons is failing to prepare backup paths.
\algoname takes a decoupled approach that sequentially plans a path for each agent.
Instead, developing coupled approaches may decrease the failure rate.

\section{Related Work}
\label{sec:related-work}

\paragraph{Path Planning for Multiple Agents}
Navigation for a team of agents are typically categorized into \emph{reactive} or \emph{deliberative} approaches.
Reactive approaches~\cite{van2011reciprocal,csenbacslar2019robust} make agents continuously react to situations at runtime to avoid collisions, while heading to their own destinations.
This class can deal with unexpected events such as crash failures.
However, provably deadlock-free systems are difficult to realize due to the shortsightedness of time evolution.
Deliberative approaches use a longer planning horizon to plan collision/deadlock-free trajectories, typically formulated as the \emph{multi-agent pathfinding (MAPF)} problem~\cite{stern2019def}.
Recent studies~\cite{atzmon2020robust,shahar2021safe,okumura2022offline} focus on robust MAPF for timing uncertainties, i.e., where agents might be delayed at runtime.
On the other hand, those studies assume that agents never crash and eventually take action; significantly different from ours.
MAPPCF is on the deliberative side but also has reactive aspects because agents change their behaviors at runtime according to failure detectors.

\algoname can be regarded as a two-level search, akin to popular MAPF algorithms~\cite{sharon2013increasing,sharon2015conflict,surynek2019unifying,lam2022branch}.
Those algorithms manage collisions at a high level, and perform single-agent pathfinding at a low level.
Instead of collision management, \algoname manages unresolved crash faults at the high level.

Multi-agent path planning with \emph{local observations} is not new in the literature~\cite{wiktor2014decentralized,zhang2016discof}.
Typically, previous studies aim at avoiding collisions or deadlocks by applying ad-hoc rules of agents' behavior, according to observation results at runtime, without assumptions of crash fault.
We note that learning-based MAPF approaches like~\cite{sartoretti2019primal} also assume local observations of each agent (e.g., field-of-view).

\paragraph{Resilient Multi-Robot Systems}
Studies on multi-robot systems sometimes assume robot crashes at runtime, e.g., for target tracking~\cite{zhou2018resilient}, orienteering~\cite{DBLP:conf/rss/ShiTZ20}, and task assignment~\cite{schwartz2020robust}.
In those studies, however, crashed robots do not disturb correct robots as we assume in this paper.
In the context of pathfinding, a few studies focus on system designs for potentially non-cooperative agents~\cite{bnaya2013multi,strawn2021byzantine} where those agents can pretend to be crashed.
However, those studies do not provide safe paths as presented in this paper.

\paragraph{Failure Detector}
The notion of a failure detector is inspired by a popular abstraction in theoretical distributed algorithms~\cite{chandra1996unreliable}, introduced to enable consensus solvability in an asynchronous setting.
With respect to the original concept, this paper assumes the detector to be both \emph{accurate} (i.e., it never suspects correct agents) and \emph{complete} (i.e., it always suspects crashed agents).
Removing these assumptions is an interesting future direction.
Also, our use of failure detectors is \emph{local}, that is, they provide localized information about failures, while previous works on failure detectors considered complete information oracles, that is, they provide \emph{global} information about \emph{all} agents.

\section{Conclusion and Discussion}
\label{sec:conclusion}
We studied a graph path planning problem for multiple agents that may crash at runtime, and block part of the workspace.
Different from conventional MAPF studies, each agent can change its executing path according to local crash failure detection;
hence a set of paths constitute a solution for each agent.
This paper presented a safe approach to ensure that correct agents reach their destinations regardless of crash patterns, including a series of theoretical analyses.
Finally, we list promising directions to extend MAPPCF.

\begin{itemize}
\item \textbf{Developing complete algorithms} that can address the crash awareness differences among agents.
\item \textbf{Optimization:}
  We focused on the feasibility problem (i.e., the decision problem of whether a given instance contains a solution); optimization problems are not addressed.
  One potential objective is minimizing the worst-case makespan (i.e., the maximum traveling time), which is convenient to practical situations.
\item \textbf{Global Failure Detector:}
  We assumed that an agent detects crashes only when it is adjacent to crashed agents.
  Extending this observation range might improve the planning success rate because agents can determine their behavior based on supplementary knowledge about crashes.
  If the observation range is the entire graph, MAPPCF becomes a centralized problem, since every agent has the same information.
\end{itemize}

\section*{Acknowledgments}
We thank the anonymous reviewers for their many insightful comments and suggestions.
This work was partly supported by JSPS KAKENHI Grant Number 20J23011, JST ACT-X Grant Number JPMJAX22A1, and ANR project SAPPORO, ref. 2019-CE25-0005-1.
The work was carried out during Keisuke Okumura's staying in LIP6, partially supported by JSPS Overseas Challenge Program for Young Researchers.
KO also thanks the support of the Yoshida Scholarship Foundation.

\bibliography{ref}
\clearpage
\appendix
\section*{Appendix}

\section{Proof of Model Power Analysis}

\begin{theorem*}[\ref{thrm:nfd-afd}; NFD and AFD]
  In SEQ, NFD is strictly stronger than AFD.
\end{theorem*}
\begin{proof}
  \Cref{fig:seq-anonymous}a is an instance that is only solvable with NFD in SEQ.
  Regardless of the failure detector types, $i$ needs to move to the center vertex as the first action (\cref{fig:seq-anonymous}b);
  otherwise, $i$ may crash at the goal vertex of another agent, and this agent never reaches its goal due to $i$'s crash.
  The same argument holds for $j$ and $k$.

  Assume that $j$ detects someone is crashed at the center vertex.
  Then $j$ needs to move to either the goal of $i$ or the goal of $k$.
  Consider that $j$ moves to $k$'s goal, and then crashes (\cref{fig:seq-anonymous}c).
  If the crashed agent at the center vertex is $i$, the remaining correct agent $k$ cannot reach its goal.
  A symmetric situation occurs if $j$ moves to $i$'s goal but the crashed agent at the center vertex is $k$.
  As a result, this instance is unsolvable with AFD.
  In contrast, with NFD, $j$ can choose the ``correct'' vertex according to detected crashes, e.g., $i$'s goal when $i$ is crashed at the center vertex.
  Even if $j$ is further crashed at $i$'s goal, $k$ can reach its goal (\cref{fig:seq-anonymous}d).
\end{proof}

{
  \newcommand{\drawedge}{
    \foreach \u / \v in {v1/v2,v1/v3,v3/v5,v4/v6,v3/v4,v5/v6,v2/v4,v1/v6,v2/v5,v2/v3,v2/v6,v3/v6}
    \draw[line] (\u) -- (\v);
  }
  \newcommand{\cross}{$\mathbin{\tikz [x=1.4ex,y=1.4ex,line width=.2ex, red] \draw (0,0) -- (1,1) (0,1) -- (1,0);}$}%
  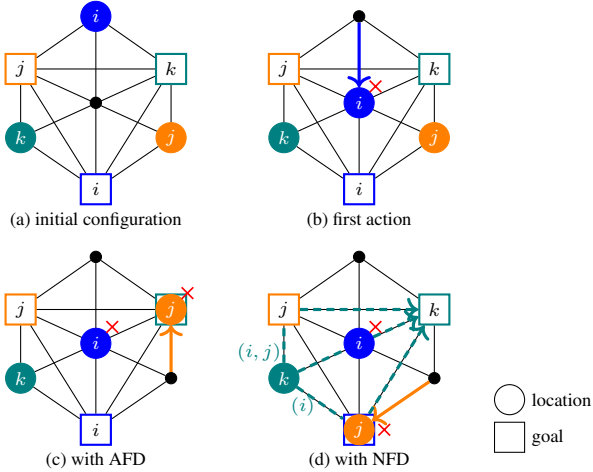
\begin{figure}[th!]
    \centering
    \scriptsize
    \begin{tikzpicture}
      \begin{scope}[shift={(0, 0)}]
        \node[start-vertex,blue,text=white](v1) at (0, 0) {$i$};
        \node[goal-vertex,draw=orange,thick](v2) at ($(v1)+(-1.0,-0.7)$) {$j$};
        \node[goal-vertex,draw=teal,thick](v3) at ($(v1)+( 1.0,-0.7)$) {$k$};
        \node[start-vertex,teal,text=white](v4) at ($(v2.south)+(0,-0.7)$) {$k$};
        \node[start-vertex,orange,text=white](v5) at ($(v3.south)+(0,-0.7)$) {$j$};
        \node[goal-vertex,draw=blue,thick](v6) at ($(v1)+(0,-2.3)$) {$i$};
        \drawedge
        \node[vertex](v7) at ($(v1)+(0,-1.15)$) {};
        \node[below=0 of v6](caption) {(a)~initial configuration};
      \end{scope}
      \begin{scope}[shift={(3.5, 0)}]
        \node[vertex](v1) at (0, 0) {};
        \node[goal-vertex,draw=orange,thick](v2) at ($(v1)+(-1.0,-0.7)$) {$j$};
        \node[goal-vertex,draw=teal,thick](v3) at ($(v1)+( 1.0,-0.7)$) {$k$};
        \node[start-vertex,teal,text=white](v4) at ($(v2.south)+(0,-0.7)$) {$k$};
        \node[start-vertex,orange,text=white](v5) at ($(v3.south)+(0,-0.7)$) {$j$};
        \node[goal-vertex,draw=blue,thick](v6) at ($(v1)+(0,-2.3)$) {$i$};
        \drawedge
        \node[start-vertex,blue,text=white](v7) at ($(v1)+(0,-1.15)$) {$i$};
        \node[above right=-0.15 of v7](crash-i) {\cross};
        \draw[line,blue,->,very thick](v1)--(v7);
        \node[below=0 of v6](caption) {(b)~first action};
      \end{scope}
      \begin{scope}[shift={(0, -3.2)}]
        \node[vertex](v1) at (0, 0) {};
        \node[goal-vertex,draw=orange,thick](v2) at ($(v1)+(-1.0,-0.7)$) {$j$};
        \node[goal-vertex,draw=teal,thick](v3) at ($(v1)+( 1.0,-0.7)$) {$k$};
        \node[start-vertex,teal,text=white](v4) at ($(v2.south)+(0,-0.7)$) {$k$};
        \node[vertex](v5) at ($(v3.south)+(0,-0.7)$) {};
        \node[goal-vertex,draw=blue,thick](v6) at ($(v1)+(0,-2.3)$) {$i$};
        \drawedge
        \node[start-vertex,blue,text=white](v7) at ($(v1)+(0,-1.15)$) {$i$};
        \node[above right=-0.15 of v7](crash-i) {\cross};
        \node[start-vertex,fill=orange,text=white,draw=orange](loc-j) at (v3) {$j$};
        \node[above right=-0.15 of loc-j](crash-j) {\cross};
        \draw[line,orange,->,very thick](v5)--(loc-j);
        \node[below=0 of v6](caption) {(c)~with AFD};
      \end{scope}
      \begin{scope}[shift={(3.5, -3.2)}]
        \node[vertex](v1) at (0, 0) {};
        \node[goal-vertex,draw=orange,thick](v2) at ($(v1)+(-1.0,-0.7)$) {$j$};
        \node[goal-vertex,draw=teal,thick](v3) at ($(v1)+( 1.0,-0.7)$) {$k$};
        \node[start-vertex,teal,text=white](v4) at ($(v2.south)+(0,-0.7)$) {$k$};
        \node[vertex](v5) at ($(v3.south)+(0,-0.7)$) {};
        \node[goal-vertex,draw=blue,thick](v6) at ($(v1)+(0,-2.3)$) {$i$};
        \drawedge
        \node[start-vertex,blue,text=white](v7) at ($(v1)+(0,-1.15)$) {$i$};
        \node[above right=-0.15 of v7](crash-i) {\cross};
        \node[start-vertex,fill=orange,text=white,draw=orange](loc-j) at (v6) {$j$};
        \node[right=-0.05 of loc-j](crash-j) {\cross};
        \draw[line,orange,->,very thick](v5)--(loc-j);
        %
        \draw[line,teal,->,densely dashed,very thick](v4)--(v7)--(v3);
        \draw[line,teal,->,densely dashed,very thick](v4)--(v2)--(v3);
        \draw[line,teal,->,densely dashed,very thick](v4)--(v6)--(v3);
        \node[color=teal] at ($(v4)+(0.4,0.05)$) {\scriptsize };
        \node[color=teal,anchor=east] at ($(v4)+(0.5,-0.37)$) {\scriptsize $(i)$};
        \node[color=teal,anchor=east] at ($(v4)+(0.05,0.35)$) {\scriptsize $(i,j)$};
        \node[below=0 of v6](caption) {(d)~with NFD};
      \end{scope}
      \begin{scope}[shift={(5.5, -5.1)}]
        \node[start-vertex, label=right:{\scriptsize location}](label-s) at (0, 0){};
        \node[goal-vertex, label=right:{\scriptsize goal}](label-g) at (0, -0.5){};
      \end{scope}
    \end{tikzpicture}
    \caption{
      Instance that is solvable with NFD, but unsolvable for AFD in SEQ.
      A red cross corresponds to a crashed agent.
      In (d), part of the plan of $k$ is visualized by dashed arrows, annotated with detected crashes.
    }
    \label{fig:seq-anonymous}
  \end{figure}
}

\section{Diode Gadgets}

\begin{observation}[diode gadget for SEQ]
Assuming SEQ, in \cref{fig:diode}b, agent $i$ cannot go back to its starting location once it has reached  vertex $u$.
\end{observation}
\begin{proof}
  Observe that agent $\alpha$ cannot move until $i$ moves to the center vertex;
  otherwise, $i$ cannot reach $u$ if $\alpha$ crashes at the center vertex.
  Once $i$ reaches $u$, $\alpha$ moves to the center vertex.
  If $\alpha$ crashes there, then $i$ cannot go back to its starting location.
  Note that regardless of the crashed status of $i$, $\alpha$ can reach its goal.
\end{proof}

\begin{observation}[diode gadget for SYN]
Assuming SYN, in \cref{fig:diode}c, agent $i$ cannot go back to its starting location once it has reached  vertex $u$.
\end{observation}
\begin{proof}
  We first show how the agent $i$ reaches the vertex $u$ in \cref{fig:diode-sync-step}.
  In the initial configuration (\cref{fig:diode}c), $\beta$ needs to wait for $i$ to move from the starting location;
  otherwise, if $\beta$ crashes there, either $i$ or $\gamma$ cannot reach its goal.
  The core observation is that $\gamma$ (or $\beta$) must use the right neighboring vertex of $i$'s starting location in Step~2 of \cref{fig:diode-sync-step}.
  Therefore, if $\gamma$ crashes at Step~2, $i$ cannot go back to the starting location.
  Note that regardless of the crashed status of $i$, $\beta$, and $\gamma$, the remaining agents can reach their goals.
\end{proof}

{
  \newcommand{\edgesizeA}{0.6}
  \newcommand{\edgesizeB}{0.8}  
  \newcommand{\edgesizeC}{0.4}  
  \newcommand{\edgesizeD}{1.5}
  \begin{figure}[th!]
    \centering
    \scriptsize
    \begin{tikzpicture}
      \begin{scope}[shift={(0, 0)}]
        \node[vertex](v1) at (0, 0) {};
        \node[start-vertex,right=\edgesizeA of v1.center](v2) {$i$};
        \node[goal-vertex,above=\edgesizeA of v2.center,thick,draw=blue](v3) {$\beta$};
        \node[goal-vertex,below=\edgesizeA of v2.center,thick,draw=orange](v4) {$\gamma$};
        \node[start-vertex,orange,text=white](v5) at ($(v3.center)+(\edgesizeB,-\edgesizeC)$) {$\gamma$};
        \node[start-vertex,blue,text=white](v6) at ($(v4.center)+(\edgesizeB,\edgesizeC)$) {$\beta$};
        \node[vertex,right=\edgesizeD of v2.center](v7) {};
        \foreach \u / \v in {v1/v2,v2/v3,v2/v4,v3/v5,v4/v6,v2/v5,v2/v6,v5/v6,v5/v7,v6/v7}
        \draw[line] (\u) -- (\v);
        \foreach \u / \v in {v3/v7,v4/v7}
        \draw[line] (\u) -| (\v);
        \draw[line,densely dotted,thick](v7)--($(v7)+(0.3,0)$);
        \draw[line,->,thick] (v1)--(v2);
        \node[] at (1.2,-1.3) {Step 1};
      \end{scope}
      \begin{scope}[shift={(2.9, 0)}]
        \node[vertex](v1) at (0, 0) {};
        \node[start-vertex,right=\edgesizeA of v1.center,orange,text=white](v2) {$\gamma$};
        \node[goal-vertex,above=\edgesizeA of v2.center,thick,draw=blue](v3) {$\beta$};
        \node[goal-vertex,below=\edgesizeA of v2.center,thick,draw=orange](v4) {$\gamma$};
        \node[start-vertex,blue,text=white](v5) at ($(v3.center)+(\edgesizeB,-\edgesizeC)$) {$\beta$};
        \node[start-vertex](v6) at ($(v4.center)+(\edgesizeB,\edgesizeC)$) {$i$};
        \node[vertex,right=\edgesizeD of v2.center](v7) {};
        \foreach \u / \v in {v1/v2,v2/v3,v2/v4,v3/v5,v4/v6,v2/v5,v2/v6,v5/v6,v5/v7,v6/v7}
        \draw[line] (\u) -- (\v);
        \foreach \u / \v in {v3/v7,v4/v7}
        \draw[line] (\u) -| (\v);
        \draw[line,densely dotted,thick](v7)--($(v7)+(0.3,0)$);
        \draw[line,->,thick] (v2)--(v6);
        \draw[line,->,thick,blue] (v6)--(v5);
        \draw[line,->,thick,orange] (v5)--(v2);
        \node[] at (1.2,-1.3) {Step 2};
      \end{scope}
      \begin{scope}[shift={(5.8, 0)}]
        \node[vertex](v1) at (0, 0) {};
        \node[vertex,right=\edgesizeA of v1.center](v2) {};
        \node[goal-vertex,above=\edgesizeA of v2.center,thick,draw=blue](v3) {$\beta$};
        \node[goal-vertex,below=\edgesizeA of v2.center,thick,draw=orange](v4) {$\gamma$};
        \node[vertex](v5) at ($(v3.center)+(\edgesizeB,-\edgesizeC)$) {};
        \node[vertex](v6) at ($(v4.center)+(\edgesizeB,\edgesizeC)$) {};
        \node[vertex,right=\edgesizeD of v2.center](v7) {};
        \node[start-vertex,blue,text=white] at (v3) {$\beta$};
        \node[start-vertex,orange,text=white] at (v4) {$\gamma$};
        \foreach \u / \v in {v1/v2,v2/v3,v2/v4,v3/v5,v4/v6,v2/v5,v2/v6,v5/v6,v5/v7,v6/v7}
        \draw[line] (\u) -- (\v);
        \foreach \u / \v in {v3/v7,v4/v7}
        \draw[line] (\u) -| (\v);
        \draw[line,densely dotted,thick](v7)--($(v7)+(0.4,0)$);
        \node[start-vertex,fill=white,text=black](loc-i) at (v7) {$i$};
        \draw[line,->,thick] (v6)--(loc-i);
        \draw[line,->,thick,blue] (v5)--(v3);
        \draw[line,->,thick,orange] (v2)--(v4);
        \node[] at (1.2,-1.3) {Step 3};
      \end{scope}
    \end{tikzpicture}
    \caption{
      Execution in the diode gadget for SYN without crashes.
    }
    \label{fig:diode-sync-step}
  \end{figure}
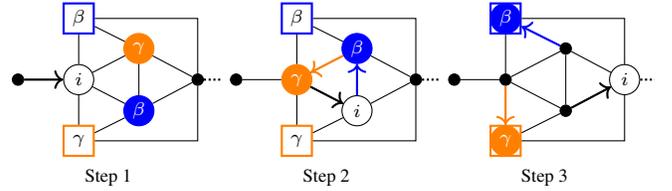
}

\section{Proof of Complexity}

\begin{theorem*}[\ref{thrm:verification}; complexity of verification]
  Verifying a solutions of MAPPCF is co-NP-complete regardless of models.
\end{theorem*}
\begin{proof}
  We show the case of SEQ with AFD.
  The proof is easily applicable to other models such as those with SYN or with NFD.
  The proof is done by reduction of the SAT problem, i.e., constructing a combination of an MAPPCF instance and a set of plans such that the plans are a solution if and only if the corresponding formula is unsatisfiable.
  Since the infeasibility check of the plan is done in polynomial time with a proper execution schedule and crash patterns, the verification is co-NP-complete.
  Throughout the proof, we use the following example:
  $(x \lor y \lor \lnot z) \land (\lnot x \lor y \lor z) \land (\lnot x \lor \lnot y)$.
  The reduction is depicted in \cref{fig:verification}.
  Let denote the number of clauses in the formula as $l$.

{
  \newcommand{\edgesizeA}{1.3}
  \newcommand{\edgesizeB}{4.7}
  \newcommand{\edgesizeC}{3.1}
  \newcommand{\edgesizeD}{1.1}
  \newcommand{\edgesizeE}{0.5}
  \newcommand{\edgesizeF}{1.8}  
  \newcommand{\edgesizeG}{0.5}  
  \newcommand{\edgesizeH}{3.5}  
  \newcommand{\edgesizeI}{0.6}  
  \newcommand{\edgesizeJ}{0.5}  
  \newcommand{\edgesizeK}{0.5}  
  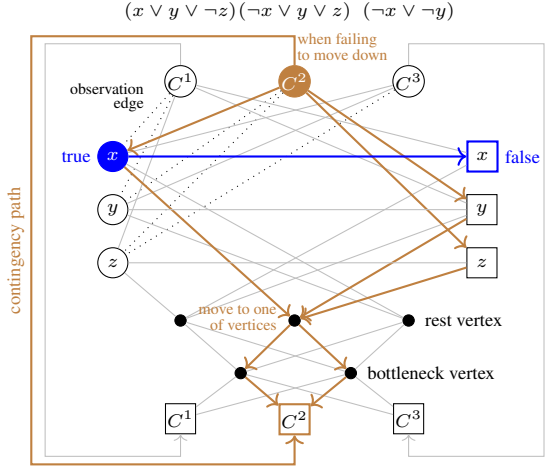
\begin{figure}[tb!]
    \centering
    \scriptsize
    \begin{tikzpicture}
      \begin{scope}[shift={(0, 0)}]
        \node[start-vertex](c1-s) at (0, 0) {$C^1$};
        \node[start-vertex,right=\edgesizeA of c1-s.center,fill=brown,text=white,draw=brown](c2-s) {$C^2$};
        \node[start-vertex,right=\edgesizeA of c2-s.center](c3-s) {$C^3$};
        %
        \node[start-vertex,blue,text=white](v1-s) at ($(c1-s) + (-0.9, -1.0)$) {$x$};
        \node[start-vertex,below=\edgesizeK of v1-s.center](v2-s) {$y$};
        \node[start-vertex,below=\edgesizeK of v2-s.center](v3-s) {$z$};
        \node[goal-vertex,right=\edgesizeB of v1-s.center,draw=blue,thick](v1-g) {$x$};
        \node[goal-vertex,right=\edgesizeB of v2-s.center](v2-g) {$y$};
        \node[goal-vertex,right=\edgesizeB of v3-s.center](v3-g) {$z$};
        \node[text=blue,anchor=east] at ($(v1-s.west)+(0,0)$) (label-true){true};
        \node[text=blue,anchor=west] at ($(v1-g.east)+(0,0)$) (label-true){false};
        %
        \node[vertex,below=\edgesizeC of c1-s.center](c1-m) {};   
        \node[vertex,below=\edgesizeC of c2-s.center](c2-m) {};
        \node[vertex,below=\edgesizeC of c3-s.center](c3-m) {};
        \node[vertex](m1) at ($(c1-m) + (0.8,-0.7)$) {};
        \node[vertex,right=\edgesizeA of m1](m2) {};
        %
        \node[goal-vertex,below=\edgesizeD of c1-m.center](c1-g) {$C^1$};
        \node[goal-vertex,below=\edgesizeD of c2-m.center,draw=brown,thick](c2-g) {$C^2$};
        \node[goal-vertex,below=\edgesizeD of c3-m.center](c3-g) {$C^3$};
        %
        \foreach \u / \v in {
          c1-s/v1-s,c1-s/v2-s,
          c2-s/v2-s,c2-s/v3-s,
          c3-s/v3-s}
        \draw[line,thin,dotted](\u)--(\v);
        %
        \foreach \u / \v in {
          c1-s/v1-g,v1-g/c1-m,c1-s/v2-g,v2-g/c1-m,c1-s/v3-s,v3-s/c1-m,
          c3-s/v1-s,v1-s/c3-m,c3-s/v2-s,v2-s/c3-m,
          c1-m/m1,c1-m/m2,m1/c1-g,m2/c1-g,
          c3-m/m1,c3-m/m2,m1/c3-g,m2/c3-g}
        \draw[line,lightgray](\u)--(\v);
        %
        \foreach \u / \v in {c2-s/v1-s,v1-s/c2-m,c2-s/v2-g,v2-g/c2-m,c2-s/v3-g,v3-g/c2-m,
          c2-m/m1,c2-m/m2,m1/c2-g,m2/c2-g}
        \draw[line,brown,->,thick](\u)--(\v);
        %
        \draw[line,lightgray,->](c1-s)
        |-($(c1-s)+(-\edgesizeF,\edgesizeG)$)
        --($(c1-g)+(-\edgesizeF,-\edgesizeG)$)
        -|(c1-g);
        \draw[line,brown,->,thick](c2-s)
        |-($(c2-s)+(-\edgesizeH,\edgesizeI)$)
        --($(c2-g)+(-\edgesizeH,-\edgesizeI)$)
        node [midway,above,rotate=90] {contingency path}
        -|(c2-g);
        \draw[line,lightgray,->](c3-s)
        |-($(c3-s)+(\edgesizeF,\edgesizeG)$)
        --($(c3-g)+(\edgesizeF,-\edgesizeG)$)
        -|(c3-g);
        %
        \draw[line,->,blue,thick](v1-s)--(v1-g);
        \draw[line,lightgray](v2-s)--(v2-g);
        \draw[line,lightgray](v3-s)--(v3-g);
        %
        \node[above=\edgesizeE of c1-s](c1-l) {$(x\lor y \lor\lnot z)$};
        \node[above=\edgesizeE of c2-s](c2-l) {$(\lnot x\lor y\lor z)$};
        \node[above=\edgesizeE of c3-s](c3-l) {$(\lnot x\lor \lnot y)$};
        %
        \node[text=brown,anchor=west] at ($(c2-s)+(-0.05,0.55)$) {\tiny when failing};
        \node[text=brown,anchor=west] at ($(c2-s)+(-0.05,0.35)$) {\tiny to move down};
        \node[text=brown,anchor=east] at ($(c2-m)+(-0.15,0.15)$) {\tiny move to one};
        \node[text=brown,anchor=east] at ($(c2-m)+(-0.15,-0.05)$) {\tiny of vertices};
        %
        \node[anchor=west,right=0.05 of c3-m]{rest vertex};
        \node[anchor=west,right=0.05 of m2]{bottleneck vertex};
        \node[anchor=east] at ($(c1-s)+(-0.4,-0.1)$) {\tiny observation};
        \node[anchor=east] at ($(c1-s)+(-0.4,-0.3)$) {\tiny edge};
      \end{scope}
    \end{tikzpicture}
    \caption{
      MAPPCF instance and solution reduced from the SAT instance $(x \lor y \lor \lnot z) \land (\lnot x \lor y \lor z) \lor (\lnot x \lor \lnot y)$.
      Dotted lines are edges that are not used in the plans.
    }
    \label{fig:verification}
  \end{figure}
}

  \smallskip
  \noindent
  \emph{A. Construction of Instance and Plans:}
  We first introduce a \emph{variable agent} for each variable $x_i$ of the formula.
  In \cref{fig:verification}, we highlight the corresponding agent of the variable $x$ as blue-colored.
  A plan for a variable agent is just to move one step to the right.

  We next introduce a \emph{clause} agent for each clause $C^j$.
  In \cref{fig:verification}, we highlight the corresponding agent of the clause $C^2 = \lnot x \lor y \lor z$ as brown-colored.
  A plan for the clause agent consists of two sub-plans:
  \begin{itemize}
    \item A \emph{primary plan} passes through \emph{(i)}~either the start or the goal of a variable agent, \emph{(ii)}~a \emph{rest vertex}, \emph{(iii)}~one of the \emph{bottleneck vertices}, and \emph{(iv)} its goal (so, four steps in total).
    \item An \emph{contingency plan} directly reaches the goal in one step.
      This plan is used when the clause agent cannot take the first step of the primary plan.
  \end{itemize}
  Each rest vertex is unique to each clause agent.
  The start of a clause agent is connected to a rest vertex via a start (or goal) of a variable agent when the clause contains a negative (resp. positive) literal of the variable.
  Bottleneck vertices are shared between clause agents, and their number is $l-1$.
  Each rest vertex and goal of a clause agent is connected to all bottleneck vertices.

  We further add \emph{observation edges} (dotted lines) between each start of a clause agent and start of a variable agent.
  These edges are not included in any plans but they enable clause agents not to enter the goals of variable agents when the variable agents are correct, and they are on their start locations.

  The translation from the formula into an MAPPCF instance and the plans is clearly done in polynomial time.

  \smallskip
  \noindent
  \emph{B. The formula is unsatisfiable when the plans are solution:}
  We can build any assignment by the execution of MAPPCF.
  Specifically, assign a variable true when the corresponding variable agent is crashed at its start; otherwise false, i.e., when the variable agent enters its goal.

  Assume now that all clause agents use their primary plan, and are currently on their rest vertex.
  If $l-1$ clause agents move to bottleneck vertices, and crash there, then the remaining agent cannot reach its goal following the plan.
  Therefore, for the set of plans to be a solution, it is necessary that at least one clause agent takes its contingency plan.
  Consider now that clause agent.
  This agent $a$ cannot take the primary plan because some variable agents are blocking $a$ either due to a crash at their start, or upon reaching their goal.
  The corresponding assignment is unsatisfiable for the clause.
  For instance, in \cref{fig:verification}, assume that the clause agent $C^2$ takes the contingency plan.
  This happens when the variable agent $x$ is crashed at its start, and both $y$ and $z$ have entered their goals.
  The assignment is then $x=\mathtt{T}$, $y=\mathtt{F}$ and $z=\mathtt{F}$, which is unsatisfiable for $C^2$.

  \smallskip
  \noindent
  \emph{C. The plans are solution when the formula is unsatisfiable:}
  When the formula is unsatisfiable, it never happens that all agents take their primary plan and all correct agents are ensured to reach their goals.
\end{proof}

\section{Failure Example}

\Cref{fig:failure-example} presents a failure planning example of \algoname in SYN.
An MAPPCF instance is shown in \cref{fig:failure-example}a.
Assume that vertices $v_1$ and $v_5$ are connected through a long path including $v_{11}$.
Observe first that this instance contains a solution (shown in~\cref{fig:failure-example}h);
those paths are vertex disjoint, hence all agents reach their goals regardless of crash patterns.

Consider now that \algoname provides initial plans as \cref{fig:failure-example}e, \ref{fig:failure-example}b, \ref{fig:failure-example}d respectively for agent $i$, $j$, and $k$ at \cref{algo:planner:init}.
Then, the planning eventually fails because there is a crash pattern that $i$ cannot prepare a backup path, as shown in \cref{fig:failure-example}g.

{
  \newcommand{\edgesize}{0.6}
  \newcommand{\drawgraph}{
    \node[vertex](v1) at (0, 0) {};
    \node[vertex](v2) at ($(v1.center)+(\edgesize,0)$) {};
    \node[vertex](v3) at ($(v1.center)+(2*\edgesize,0)$) {};
    \node[vertex](v4) at ($(v1.center)+(3*\edgesize,0)$) {};
    \node[vertex](v5) at ($(v1.center)+(4*\edgesize,0)$) {};
    \node[vertex](v6) at ($(v1.center)+(\edgesize,2*\edgesize)$) {};
    \node[vertex](v7) at ($(v1.center)+(\edgesize,\edgesize)$) {};
    \node[vertex](v8) at ($(v1.center)+(\edgesize,-\edgesize)$) {};
    \node[vertex](v9) at ($(v1.center)+(3*\edgesize,\edgesize)$) {};
    \node[vertex](v10) at ($(v1.center)+(3*\edgesize,-\edgesize)$) {};
    \node[vertex](v11) at ($(v1.center)+(2*\edgesize,-2*\edgesize)$) {};
    \foreach \u / \v in {v1/v5,v6/v8,v9/v10,v1/v7,v1/v8}
    \draw[line] (\u) -- (\v);
    \draw[line,densely dashed] (v1)|-(v11)-|(v5);
  }
  \newcommand{\cross}{$\mathbin{\tikz [x=1.4ex,y=1.4ex,line width=.2ex, red] \draw (0,0) -- (1,1) (0,1) -- (1,0);}$}
  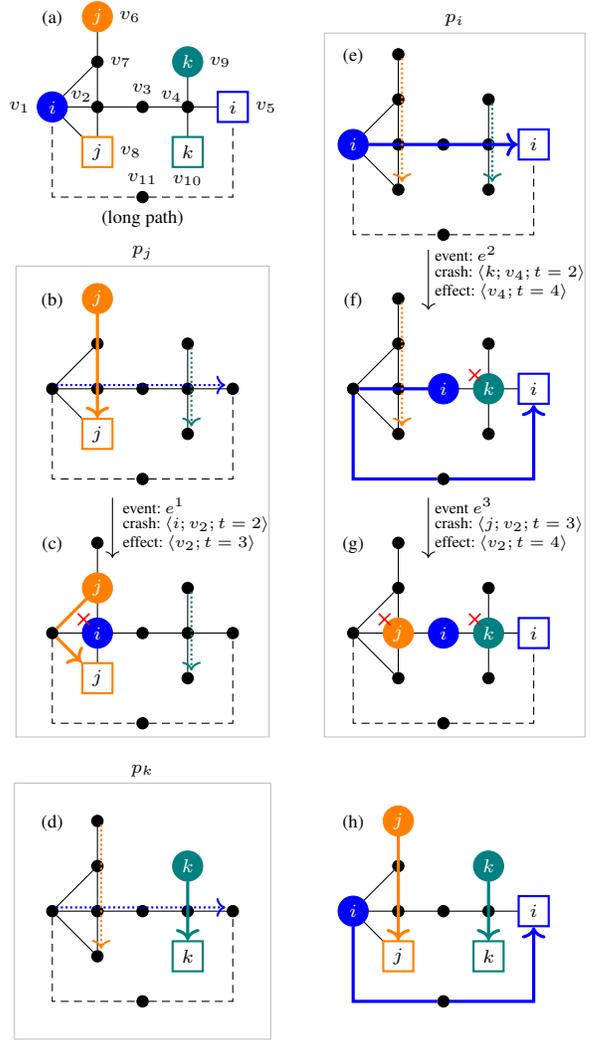
\begin{figure}[tb!]
    \centering
    \scriptsize
    \begin{tikzpicture}
      \begin{scope}[shift={(0, 0)}]
        \drawgraph
        \node[start-vertex,blue,text=white,label=left:{$v_1$}] at (v1) {$i$};
        \node[goal-vertex,draw=blue,thick,label=right:{$v_5$}] at (v5) {$i$};
        \node[start-vertex,orange,text=white,label=right:{$v_6$}] at (v6) {$j$};
        \node[goal-vertex,draw=orange,thick,label=right:{$v_8$}] at (v8) {$j$};
        \node[start-vertex,teal,text=white,label=right:{$v_9$}] at (v9) {$k$};
        \node[goal-vertex,draw=teal,thick,label=below:{$v_{10}$}] at (v10) {$k$};
        \node[above left=-0.1 of v2] {$v_2$};
        \node[above=0 of v3] {$v_3$};
        \node[above left=-0.1 of v4] {$v_4$};
        \node[right=0 of v7] {$v_7$};
        \node[above=0 of v11] {$v_{11}$};
        \node[below=0 of v11] {(long path)};
        \node[above=0.9 of v1] {(a)};
      \end{scope}
      \begin{scope}[shift={(4, -0.5)}]
        \drawgraph
        \node[start-vertex,blue,text=white](s) at (v1) {$i$};
        \node[goal-vertex,draw=blue,thick](g) at (v5) {$i$};
        \draw[line,->,very thick,blue](s)--(g);
        \draw[line,->,orange,thick,densely dotted]($(v6)+(0.05,-0.05)$)--($(v8)+(0.05,0.1)$);
        \draw[line,->,teal,thick,densely dotted]($(v9)+(0.05,-0.05)$)--($(v10)+(0.05,0.1)$);
        \node[above=0.9 of v1] {(e)};
      \end{scope}
      \begin{scope}[shift={(4, -3.75)}]
        \drawgraph
        \node[start-vertex,blue,text=white](s) at (v3) {$i$};
        \node[goal-vertex,draw=blue,thick](g) at (v5) {$i$};
        \node[start-vertex,teal,text=white](k) at (v4) {$k$};
        \node[above left=-0.2 of k] {\cross};
        \draw[line,->,very thick,blue](s)--(v1)|-(v11)-|(g);
        \draw[line,->,orange,thick,densely dotted]($(v6)+(0.05,-0.05)$)--($(v8)+(0.05,0.1)$);
        \node[above=0.9 of v1] {(f)};
      \end{scope}
      \begin{scope}[shift={(4, -7)}]
        \drawgraph
        \node[start-vertex,blue,text=white](s) at (v3) {$i$};
        \node[goal-vertex,draw=blue,thick](g) at (v5) {$i$};
        \node[start-vertex,teal,text=white](k) at (v4) {$k$};
        \node[start-vertex,orange,text=white](j) at (v2) {$j$};
        \node[above left=-0.2 of j] {\cross};
        \node[above left=-0.2 of k] {\cross};
        \node[above=0.9 of v1] {(g)};
      \end{scope}
      \begin{scope}[shift={(0, -3.75)}]
        \drawgraph
        \node[start-vertex,orange,text=white](s) at (v6) {$j$};
        \node[goal-vertex,draw=orange,thick](g) at (v8) {$j$};
        \draw[line,->,very thick,orange](s)--(g);
        \draw[line,->,blue,thick,densely dotted]($(v1)+(0.05,0.05)$)--($(v5)+(-0.1,0.05)$);
        \draw[line,->,teal,thick,densely dotted]($(v9)+(0.05,-0.05)$)--($(v10)+(0.05,0.1)$);
        \node[above=0.9 of v1] {(b)};
      \end{scope}
      \begin{scope}[shift={(0, -7.0)}]
        \drawgraph
        \node[start-vertex,orange,text=white](s) at (v7) {$j$};
        \node[goal-vertex,draw=orange,thick](g) at (v8) {$j$};
        \node[start-vertex,blue,text=white](i) at (v2) {$i$};
        \node[above left=-0.2 of i] {\cross};
        \draw[line,->,very thick,orange](s)--(v1)--(g);
        \draw[line,->,teal,thick,densely dotted]($(v9)+(0.05,-0.05)$)--($(v10)+(0.05,0.1)$);
        \node[above=0.9 of v1] {(c)};
      \end{scope}
      \begin{scope}[shift={(0, -10.7)}]
        \drawgraph
        \node[start-vertex,teal,text=white](s) at (v9) {$k$};
        \node[goal-vertex,draw=teal,thick](g) at (v10) {$k$};
        \draw[line,->,very thick,teal](s)--(g);
        \draw[line,->,blue,thick,densely dotted]($(v1)+(0.05,0.05)$)--($(v5)+(-0.1,0.05)$);
        \draw[line,->,orange,thick,densely dotted]($(v6)+(0.05,-0.05)$)--($(v8)+(0.05,0.1)$);
        \node[fit=(v1)(v5)(v6)(v11), draw=lightgray,inner sep=12pt,label=above:{$\P_k$}] {};
        \node[above=0.9 of v1] {(d)};
      \end{scope}
      \begin{scope}[shift={(4, -10.7)}]
        \drawgraph
        \node[start-vertex,blue,text=white](s-i) at (v1) {$i$};
        \node[goal-vertex,draw=blue,thick](g-i) at (v5) {$i$};
        \node[start-vertex,orange,text=white](s-j) at (v6) {$j$};
        \node[goal-vertex,draw=orange,thick](g-j) at (v8) {$j$};
        \node[start-vertex,teal,text=white](s-k) at (v9) {$k$};
        \node[goal-vertex,draw=teal,thick](g-k) at (v10) {$k$};
        \draw[line,->,very thick,blue](s-i)|-(v11)-|(g-i);
        \draw[line,->,very thick,orange](s-j)--(g-j);
        \draw[line,->,very thick,teal](s-k)--(g-k);
        \node[above=0.9 of v1] {(h)};
      \end{scope}
      \coordinate[](c2) at (3.7,0.9);
      \coordinate[](c3) at (7.0,-8.3);
      \node[fit=(c2)(c3),draw=lightgray,label=above:{$\P_i$}] {};
      \draw[line,->](5.0,-1.9)--(5.0,-2.7);
      \node[anchor=west] at (5.0,-1.95) {\tiny event: $e^2$};
      \node[anchor=west] at (5.0,-2.2)  {\tiny crash: $\langle k; v_4; t=2 \rangle$};
      \node[anchor=west] at (5.0,-2.45) {\tiny effect: $\langle v_4; t=4 \rangle$};
      \draw[line,->](5.0,-5.2)--(5.0,-6.0);
      \node[anchor=west] at (5.0,-5.3) {\tiny event $e^3$};
      \node[anchor=west] at (5.0,-5.55){\tiny crash: $\langle j; v_2; t=3 \rangle$};
      \node[anchor=west] at (5.0,-5.8) {\tiny effect: $\langle v_2; t=4 \rangle$};
      \coordinate[](c2) at (-0.4,-2.2);
      \coordinate[](c3) at (2.8,-8.3);
      \node[fit=(c2)(c3),draw=lightgray,label=above:{$\P_j$}] {};
      \draw[line,->](0.8,-5.2)--(0.8,-6.0);
      \node[anchor=west] at (0.85,-5.3) {\tiny event: $e^1$};
      \node[anchor=west] at (0.85,-5.55){\tiny crash: $\langle i; v_2; t=2 \rangle$};
      \node[anchor=west] at (0.85,-5.8) {\tiny effect: $\langle v_2; t=3 \rangle$};
    \end{tikzpicture}
    \caption{Failure example of \algoname in SYN.}
    \label{fig:failure-example}
  \end{figure}
}

\section{Further Experimental Results}

We additionally present two empirical results: the effect of the refinement over the initial plans and the result in larger grids.
\Cref{fig:grids} shows the used grids in the experiment.

{
  \newcommand{\colsize}{0.25\linewidth}
  \newcommand{\imgwidth}{0.7\linewidth}
  \setlength{\tabcolsep}{0pt}
  \begin{figure}[ht!]
    \begin{tabular}{cccc}
      \begin{minipage}{\colsize}
        \centering
            {\tiny \mapname{random-32-32-10}}
            \\\vspace{-0.15cm}{\tiny 32x32; $|V|$=922}\\
            \includegraphics[width=\imgwidth]{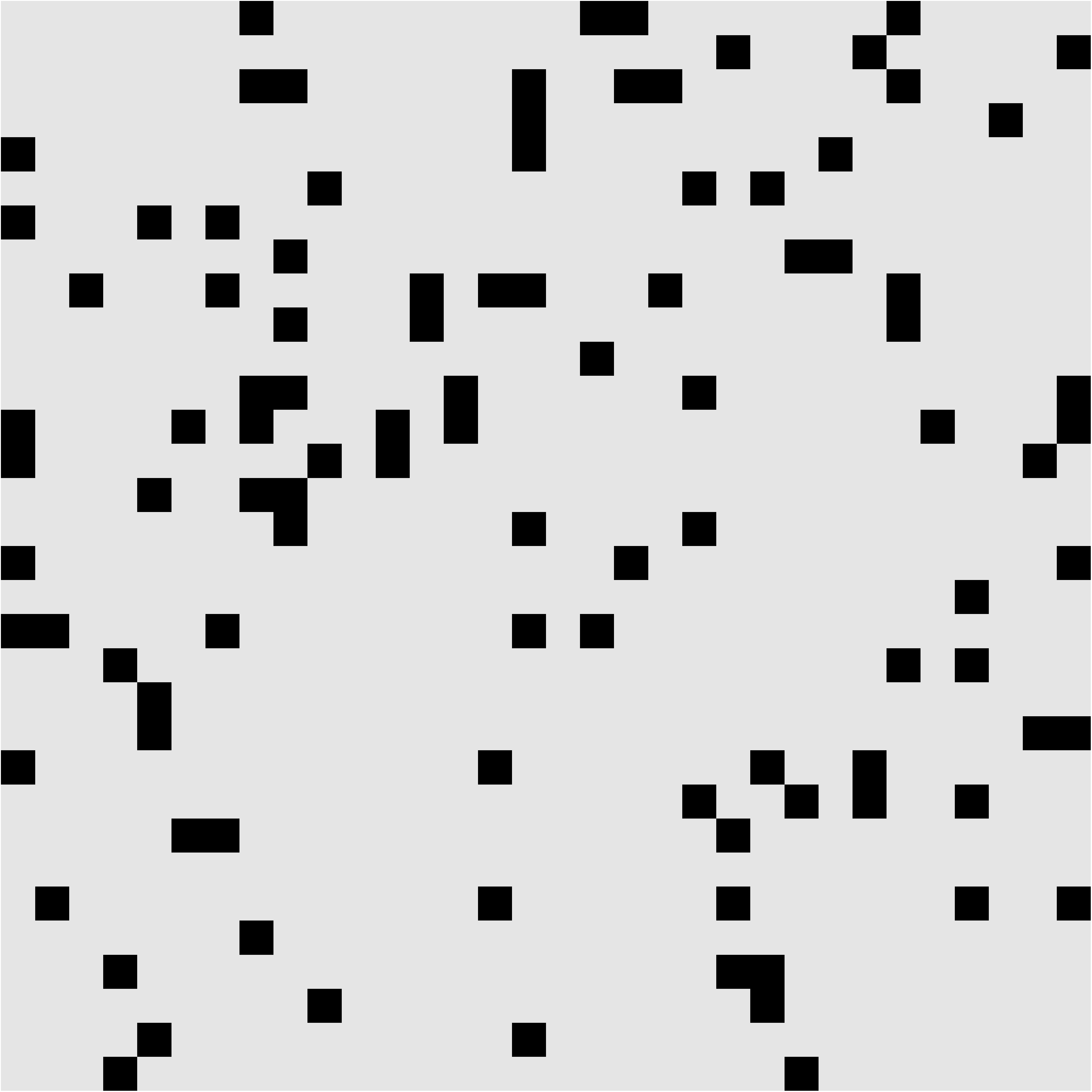}
      \end{minipage}
      &
      \begin{minipage}{\colsize}
        \centering
            {\tiny \mapname{random-64-64-10}}
            \\\vspace{-0.15cm}{\tiny 64x64; $|V|$=3,687}\\
            \includegraphics[width=\imgwidth]{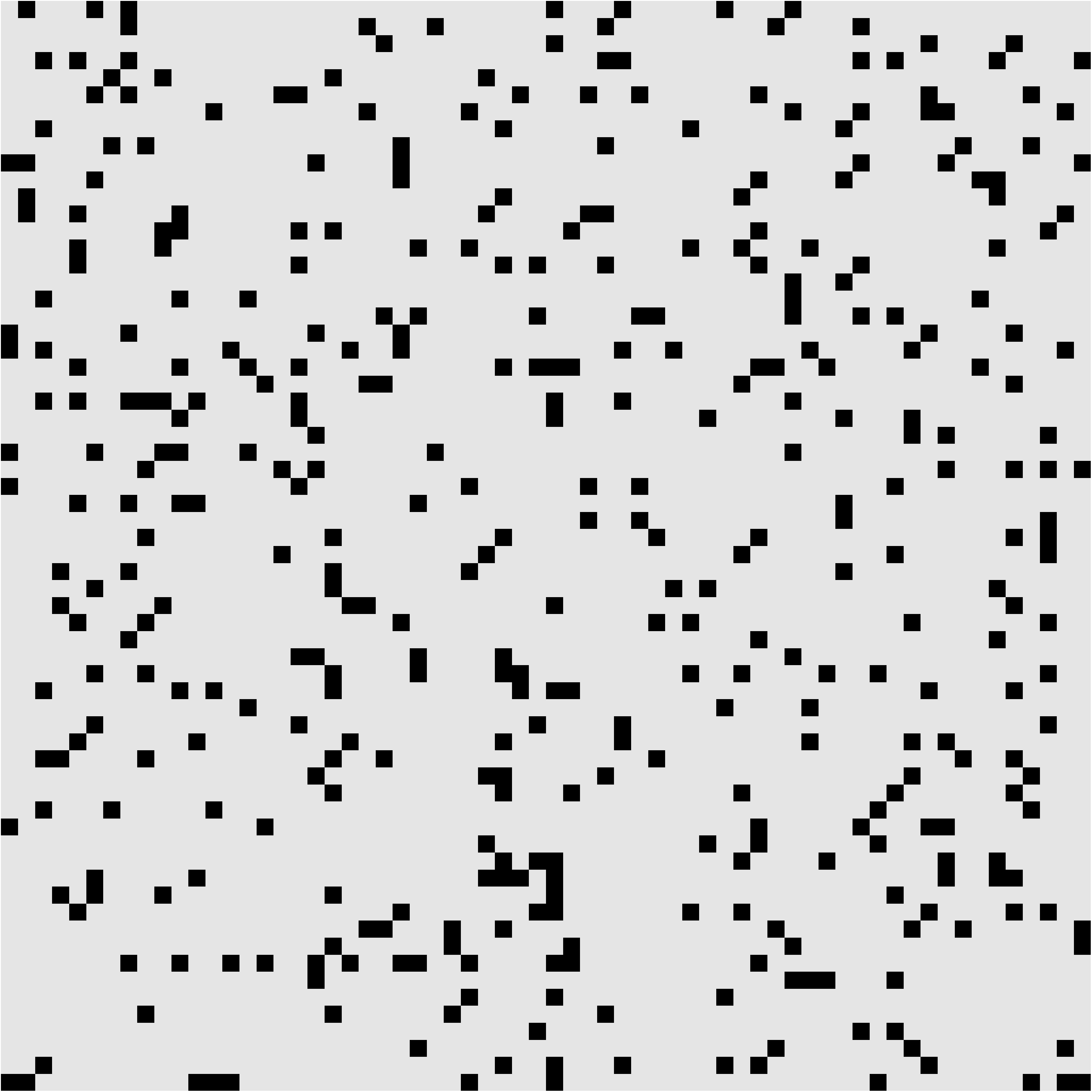}
      \end{minipage}
      &
      \begin{minipage}{\colsize}
        \centering
            {\tiny \mapname{warehouse-20-40-10-2-2}}
            \\\vspace{-0.15cm}{\tiny 340x164; $|V|$=38,756}\\
            \includegraphics[width=\imgwidth]{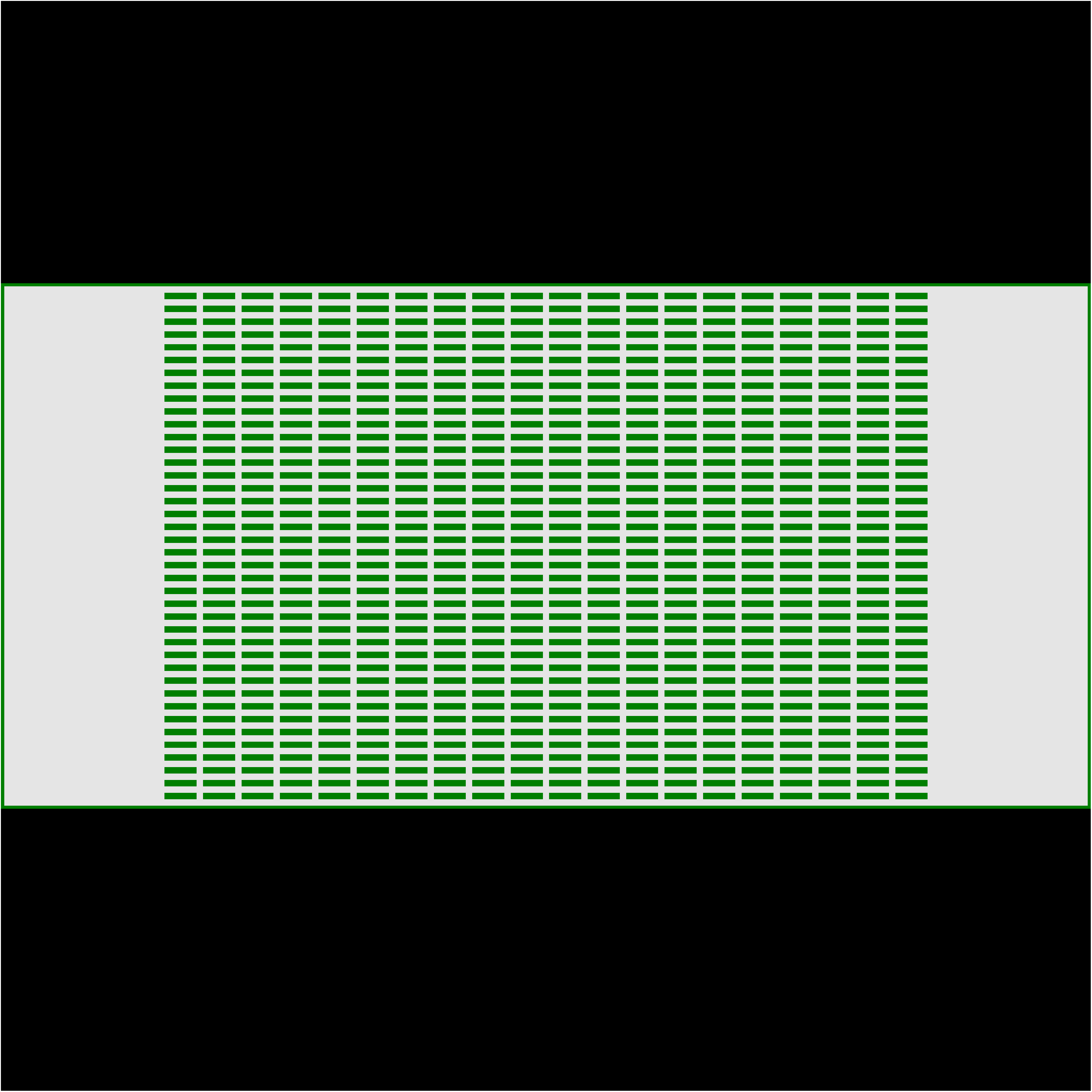}
      \end{minipage}
      &
      \begin{minipage}{\colsize}
        \centering
            {\tiny \mapname{Paris\_1\_256}}
            \\\vspace{-0.15cm}{\tiny 256x256; $|V|$=47,240}\\
            \includegraphics[width=\imgwidth]{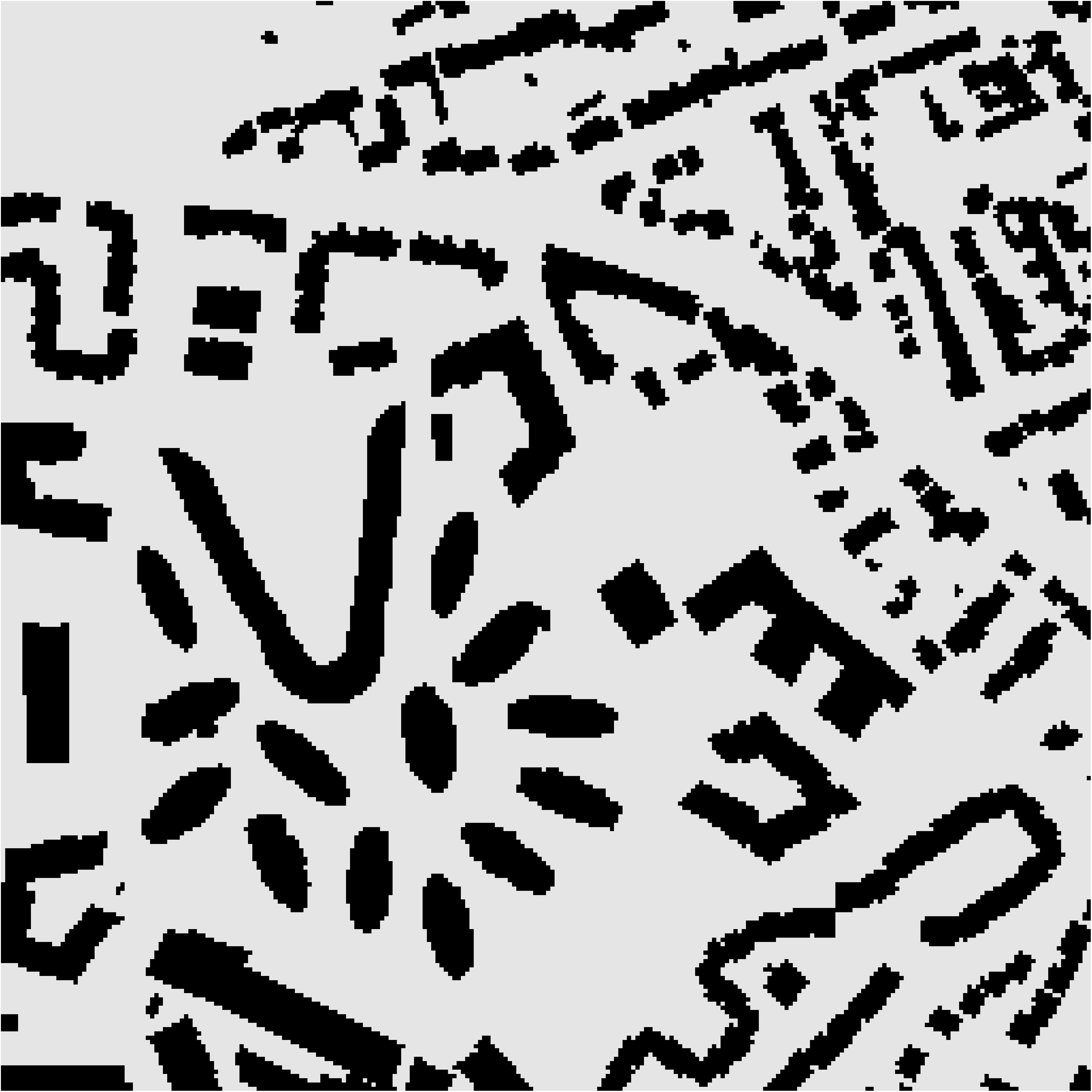}
      \end{minipage}
    \end{tabular}
    \caption{Used grids in the experiments.}
    \label{fig:grids}
  \end{figure}
}

\subsection{Effect of Refinement}
Here, we complement the effect of refinement over the initial paths introduced in \cref{sec:impl}.
\Cref{fig:result-refine} presents the number of solved instances given a certain time, over the identical instances to those of \cref{fig:result-main} (the scenario of fixed crashes).
The timeout was set to \SI{30}{\second}.
We can see a slight effect in SYN while not so in SEQ.
This is owing to that finding deadlock-free (backup) path itself is difficult.

{
  \newcommand{\colsize}{0.46\linewidth}
  \setlength{\tabcolsep}{1pt}
  \newcommand{\block}[2]{
    \begin{minipage}{\colsize}
      \centering
      \includegraphics[width=1.0\linewidth]{fig/raw/cactus_#1_#2_fix_crash.pdf}
    \end{minipage}
  }
  \begin{figure}[th!]
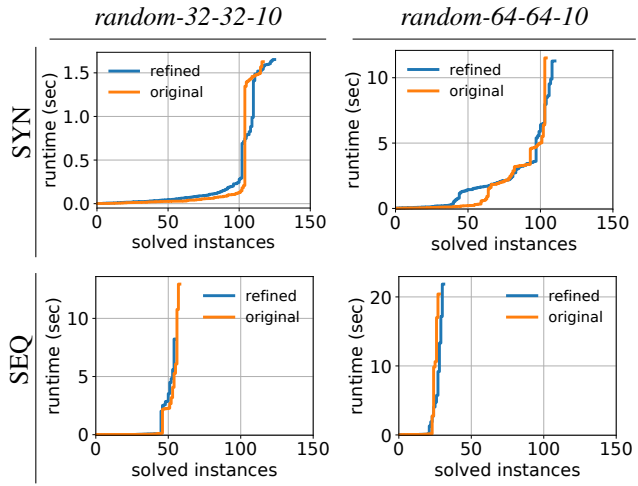

    \begin{tabular}{ccc}
      &
      \mapname{random-32-32-10} &
      \mapname{random-64-64-10}
      \\\cmidrule(lr){2-2}\cmidrule(lr){3-3}
      \multicolumn{1}{c|}{\rotatebox{90}{SYN}} &
      \block{random-32-32-10}{sync} &
      \block{random-64-64-10}{sync}
      \medskip\\
      \multicolumn{1}{c|}{\rotatebox{90}{SEQ}} &
      \block{random-32-32-10}{seq}&
      \block{random-64-64-10}{seq}
    \end{tabular}
    \caption{
      \textbf{Effect of refinement.}
      The number of solved instances until a certain time is visualized.
      The identical instances to those of \cref{fig:result-main} were used.
      $f$ was fixed to one while the number of agents varied from 5 to 30 in \mapname{random-32-32-10} and from 10 to 60 in \mapname{random-64-64-10}.
    }
    \label{fig:result-refine}
  \end{figure}
}

\subsection{Results of Large Instances}
We evaluated \algoname in larger grids compared to the experiment in \cref{sec:evaluation}.
We fixed the number of crashes $f$ as one while changing the number of agents up to $80$.
25 well-formed instances in \mapname{warehouse-20-40-10-2-2} and \mapname{Paris\_1\_256} were prepared.

\Cref{table:result-large} presents the success rate of \algoname with the \SI{5}{\minute} timeout.
In SYC, \algoname solved a moderate number of instances even with tens of agents, while in SEQ, \algoname failed to most instances.
These gaps stem from the difficulty of finding deadlock-free paths in SEQ.

{
  \begin{table}[th!]
    \centering
    \begin{tabular}{rrrrrr}
      \toprule
      & $|A|$ & 20 & 40 & 60 & 80 \\
      \midrule
      \mapname{warehouse-}
      & SYN & 1.00 & 1.00 & 1.00 & 0.76 \\
      \mapname{20-40-10-2-2}
      & SEQ & 0.64 & 0.00 & 0.00 & 0.00
      \\\cmidrule(lr){1-6}
      \multirow{2}{*}{\mapname{Paris\_1\_256}}
      & SYN & 1.00 & 0.84 & 0.52 & 0.04 \\
      & SEQ & 0.24 & 0.00 & 0.00 & 0.00 \\
      \bottomrule
    \end{tabular}
    \caption{
      \textbf{Success rate of large instances.}
      The number of crashes is fixed as $f=1$.
    }
    \label{table:result-large}
  \end{table}
}

\end{document}